\documentclass{article}

\usepackage{microtype}
\usepackage{graphicx}
\usepackage{subfigure}
\usepackage{booktabs} % for professional tables

\usepackage{hyperref}

\usepackage{bm}
\usepackage{multirow}
\usepackage{amssymb}
\usepackage{wrapfig}
\usepackage{url}
\usepackage{textcomp}
\usepackage{gensymb}

\usepackage{enumitem}

\usepackage{amsmath}
\usepackage{amsthm,amsmath,amssymb}
\usepackage{color}

\newtheorem{assumption}{Assumption}
\newtheorem{theorem}{Theorem}

\newtheorem{definition}{Definition}
\usepackage{subfigure}
\usepackage{lipsum}

\usepackage{cleveref}

\makeatletter
\newenvironment{smalleralign}[1][\small]
 {\par\nopagebreak\leavevmode\vspace*{-\baselineskip}%
  \skip0=\abovedisplayskip
  #1%
  \def\maketag@@@##1{\hbox{\m@th\normalfont\normalsize##1}}%
  \abovedisplayskip=\skip0
  \align}
 {\endalign\ignorespacesafterend}
\makeatother

\usepackage[accepted]{icml2021}

\icmltitlerunning{Multi-Objective Meta Learning}

\begin{document}

\twocolumn[

\icmltitle{Multi-Objective Meta Learning}
% \icmltitle{A Gradient Based Multiobjective Bilevel Optimization Framework for Meta Learning}



\icmlsetsymbol{equal}{*}

\begin{icmlauthorlist}
\icmlauthor{Feiyang Ye}{equal,to}
\icmlauthor{Baijiong Lin}{equal,to}
\icmlauthor{Zhixiong Yue}{to}
\icmlauthor{Pengxin Guo}{to}
\icmlauthor{Qiao Xiao}{to}
\icmlauthor{Yu Zhang}{to}
\end{icmlauthorlist}

\icmlaffiliation{to}{Department of Computer Science and Engineering, Southern University of Science and Technology, Shenzhen, China}

\icmlcorrespondingauthor{Yu Zhang}{yu.zhang.ust@gmail.com}

\icmlkeywords{Machine Learning, ICML}

\vskip 0.3in
]



\printAffiliationsAndNotice{\icmlEqualContribution} % otherwise use the standard text.

\begin{abstract}
% An important goal in artificial intelligence is learning to quickly adapt to new tasks. Meta learning, which aims to achieve this goal, thus has attracted much attention. In real applications, there usually exist multiple targets to be considered. Thus meta learning with multiple objectives has been studied recently. However, existing studies either apply an inefficient evolutionary algorithm or linearly combine multiple objectives as a single-objective problem with the need to tune combination weights, which is not an efficient method. In this paper, we propose a unified gradient-based Multi-Objective Meta Learning (MOML) framework, which is formulated as a multi-objective bi-level problem where the upper-level subproblem is to solve multiple objectives for the meta learner. To solve the MOML framework, a gradient-based optimization algorithm is proposed. Theoretically, we prove the convergence properties of the proposed gradient-based optimization algorithm. Empirically, we show the effectiveness of the proposed MOML framework in several meta learning problems, including few-shot learning, neural architecture search, domain adaptation, and multi-task learning.

Meta learning with multiple objectives can be formulated as a Multi-Objective Bi-Level optimization Problem (MOBLP) where the upper-level subproblem is to solve several possible conflicting targets for the meta learner. However, existing studies either apply an inefficient evolutionary algorithm or linearly combine multiple objectives as a single-objective problem with the need to tune combination weights. In this paper, we propose a unified gradient-based Multi-Objective Meta Learning (MOML) framework and devise the first gradient-based optimization algorithm to solve the MOBLP by alternatively solving the lower-level and upper-level subproblems via the gradient descent method and the gradient-based multi-objective optimization method, respectively. Theoretically, we prove the convergence properties of the proposed gradient-based optimization algorithm. Empirically, we show the effectiveness of the proposed MOML framework in several meta learning problems, including few-shot learning, neural architecture search, domain adaptation, and multi-task learning.
\end{abstract}

\section{Introduction}

In the past few years, deep learning has achieved great success in various fields \cite{pouyanfar2018survey} because it can effectively and efficiently process massive and high-dimensional data. However, training a deep learning model from scratch often requires a large amount of data to learn a large number of model parameters and needs to choose hyperparameters by hand, leading to a huge dependence of data volume and the choice of hyperparameters.
% leading to the risk of overfitting and the poor generalization.

As one way to address those problems by enabling models to learn how to learn, meta learning has attracted considerable attention recently \cite{hospedales2020meta, huisman2020survey}. Meta learning gains knowledge from multiple meta training tasks so that the knowledge can be reused in new tasks or new environments rapidly with a few training examples. Taken broadly, objective functions of meta learning models are usually formulated as a bi-level optimization problem where the lower-level subproblem represents the adaptation to a given task with learned meta parameters and the upper-level subproblem tries to optimize these meta parameters via a meta objective \cite{hospedales2020meta}. Hence, from this view, meta learning has a wide range of applications such as hyperparameter optimization \cite{franceschi2018bilevel}, Neural Architecture Search (NAS) \cite{lsy19}, and Reinforcement Learning (RL) \cite{xu2018meta}.

In many studies on conventional meta learning methods and applications, there is only a single meta objective in the upper-level subproblem. For example, the Model-Agnostic Meta-Learning (MAML) method \cite{finn2017model} only measures the performance on a validation dataset in the upper-level subproblem to evaluate the learned initialization of parameters. DARTS \cite{lsy19}, a differentiable method for NAS, evaluates the performance of the searched architecture on the validation dataset. However, in real world applications, there are usually more than one objective to be considered. For example, for MAML, we may need to consider not only the performance but also the robustness which can help adapt to new tasks with the learned initialization. Similarly, the network size and performance should be balanced in NAS especially when the searched architecture will be deployed to devices with limited resource such as mobile phones. In those applications, we can see that there is a need to balance multiple conflicting objectives in meta learning.

Meta learning with multiple objectives thus has drawn much attention in recent studies. Specifically, some works study specific meta learning problems in the multi-objective case, such as multi-objective NAS \cite{wu2019fbnet, tcpvshl19, cai2019once, lu2020nsganetv2},  multi-objective RL \cite{chen2019meta}, and so on. However, those works simply combine multiple objectives into a single objective with tuned weights or utilize time-consuming evolutionary algorithms that cannot be integrated into gradient-based learning models such as deep neural networks. Objective functions in meta learning with multiple objectives are naturally formulated as a Multi-Objective Bi-Level optimization Problem (MOBLP) where the lower-level subproblem is to learn the adaptation to a task similar to vanilla meta learning and the upper-level subproblem contains multiple objectives for the meta learner. In the optimization community, some works \cite{deb2009solving,sinha2011bilevel, ruuska2012constructing} study the optimization of the MOBLP but with evolutionary algorithms whose convergence property is unclear. Therefore, meta learning with multiple objectives lacks a gradient-based solution
with convergence guarantee.

To fill this gap, in this paper we propose a unified gradient-based Multi-Objective Meta Learning (MOML) framework with convergence guarantee. The MOML framework is formulated as a MOBLP, where the upper-level subproblem is to solve multiple objectives for the meta learner. We devise the first gradient-based optimization algorithm to solve the MOBLP by alternatively solving the lower-level and upper-level subproblems via the gradient descent method and the gradient-based multi-objective optimization method such as MGDA \cite{desideri12}. The convergence properties of this gradient-based method have been proved. To show the effectiveness of the MOML framework, we apply it to several meta learning problems, including few-shot learning, NAS, domain adaptation, and multi-task learning.

The main contributions of this paper are four-fold.
\begin{itemize}[noitemsep,topsep=0pt,parsep=0pt,partopsep=0pt]
\item We propose a unified MOML framework based on the MOBLP and devise a gradient-based optimization algorithm for the MOML framework.
\item We prove the convergence property of the proposed optimization algorithm.
\item We formulate several learning problems as instances of the MOML framework.
\item Experiments on those learning problems show the effectiveness of the MOML framework.%  show ove that this gradient based approach is also useful in bi-level cases, which also support our theory
%	\item We first show the domain adaption problem can be formulated as a multi-objective meta learning problem.
\end{itemize}

\section{Related Work}

\subsection{Meta Learning}

Meta learning, or learning to learn, learns knowledge from multiple tasks and then adapts it to new tasks with a few samples quickly. Many studies in meta learning mainly focus on solving the few-shot learning problem. From this view, meta learning can be divided into three main categories, including metric-based approach \cite{snell2017prototypical, sung2018learning}, model-based approach \cite{li2016learning}, and optimization-based approach \cite{finn2017model, nichol2018first}. For example, as an optimization-based method, MAML learns an initialization of model parameters so that a new task can be learned with a few training samples by fine-tuning the learned initialization. %Prototypical Networks \cite{snell2017prototypical}, a representative metric-based model, uses a prototype vector to represent each class in the embedding space and then makes classification according to the distance between embedded test data and prototype vectors.\footnote{***Where is the introduction of a model-based method?}

A widely-used formulation in meta learning can be cast as a bi-level optimization problem, where the upper-level subproblem is to learn meta parameters according to the meta objectives and the lower-level subproblem is to quickly adapt to new tasks with meta parameters \cite{rajeswaran2019meta,hospedales2020meta}. For example, MAML adapts to a new task by using the associated training dataset and the learned initialization in the lower level, and then updates the initialization according to the validation performance in the upper level. From this perspective, meta learning is a general learning paradigm and has more general applications \cite{hospedales2020meta}.
% , such as hyperparameter optimization \cite{franceschi2018bilevel}, NAS \cite{lsy19}, RL and so on.
In this paper, we study meta learning from the perspective of the bi-level optimization.% so that our proposed multi-objective meta learning framework is applicable to any bi-level optimization problems.
%Besides, we

\subsection{Multi-Objective Optimization}

Multi-objective optimization is to address the problem of optimizing multiple targets simultaneously. Actually, machine learning algorithms often have to simultaneously achieve multiple targets, which may have conflicts with each other. For example, when we train a deep learning model, minimizing the sizes of model parameters and maximizing the classification accuracy are two conflicting objectives. Many machine learning algorithms deal with multiple objectives by simply aggregating them so that the multi-objective optimization problem reduces to a single-objective optimization problem. Recently, with the development of multi-objective optimization algorithms, they have been successfully applied to solve machine learning problems \cite{jin2008pareto}.%it is shown that applying Pareto-based multi-objective optimization to machine learning can not only improve the performance than weighted sum as single-objective optimization but also achieve a better trade-off among multiple targets \cite{jin2008pareto}.

There are many kinds of multi-objective optimization algorithms, such as evolutionary algorithms \cite{zhou2011multiobjective}, population-based algorithms \cite{giagkiozis2015overview}, gradient-based algorithms \cite{desideri12, mahapatra2020multi}, and so on. In this work, we focus on gradient-based algorithms because this approach can be easily integrated into gradient-based machine learning models such as deep learning models. A representative method is the Multiple Gradient Descent Algorithm (MGDA) \cite{desideri12}, which leverages multi-objective Karush-Kuhn-Tucker (KKT) conditions \cite{kuhn2014nonlinear} and finds a common direction to decrease all the objectives.  %Recently, Sener and Koltun \cite{sk18} formulated multi-task learning as a multi-objective optimization problem and applied MGDA to solve it.

\section{The MOML Framework}

In this section, we introduce the proposed MOML framework.

\subsection{Notations and Terminologies}

We first define some useful notations and terminologies.

For a multi-objective optimization problem with $m$ objectives, each objective function is denoted by $g_i: \mathbb{R}^n \to  \mathbb{R}$, where the solution space is in the $n$-dimensional space. By combining the $m$ objectives, the resultant vector-valued function $g: \mathbb{R}^n \to  \mathbb{R}^m$ is a mapping from the solution space $\mathbb{R}^n$ to the objective space $\mathbb{R}^m$. A Multi-Objective optimization Problem (MOP) is usually formulated as
\begin{equation} \label{generalMOP}
    \min_z g(z) = (g_1(z),...g_m(z))^T\ \  \mathrm{s.t.} \ z\in \mathcal{Z},
\end{equation}
with $g: \mathbb{R}^n \to  \mathbb{R}^m$ and a nonempty set $\mathcal{Z}\subseteq \mathbb{R}^n$.

Let $P = \mathbb{R}^m_{+}$ be a pointed, closed and convex cone. Then this cone $P$ induces a partial order relation $\le_P$ in $\mathbb{R}^m$. For $l^1,l^2\in \mathbb{R}^m$, the partial ordering $l^1 \le_P l^2$ implies that $l^1_i\le l^2_i$ for all $i\in \{1,...,m\}$, where $l^1_i$ and $l^2_i$ denote the $i$th entry in $l^1$ and $l^2$, respectively.
%If $l^1 \le_P l^2$, then $l^2\in l^1+P$.
The strict inequality $l^1<_P l^2$ holds when $l^1_i<l^2_i$ holds for at least one $i$.

In the MOP, a point $l^1 \in C$ is said to be dominated by another point $l^2$ iff $l^2 \le_P l^1$. $l^2 \nless_P l^1$ means that $l^1$ is not dominated by $l^2$.
A point in a set $C\subseteq \mathbb{R}^m$ is a minimal point if it is not dominated by any other points in $C$. Therefore, the set of all minimal points in $C$ w.r.t. the ordering cone $P$ is defined as
$$\mathrm{Min}\ C := \{l^* \in C : \forall l\in C/\{l^*\}, l \nless_P l^* \}.$$
We denote by $\mathrm{Min}\ g(z)$ the set of all the minimal points of a vector-valued function $g$. We also call it as the Pareto frontier or Pareto-optimal set. Thus, the corresponding efficient solution or Pareto-optimal solution of $g(z)$ can be defined as
\begin{equation*}
    \mathrm{Eff}\ (g(z)) := \{z\in \mathcal{Z}:g(z)\in \mathop{\mathrm{Min}}_{z\in\mathcal{Z}}g(z)\}.
\end{equation*}

%\subsection{A multi-object bi-level optimization framework}
\subsection{Formulation}

The proposed MOML framework has a unified objective function, which is formulated as a MOBLP, as
\begin{equation}\label{P1}
    \min_{\alpha\in \mathcal{A}, \omega \in \mathbb{R}^p} F(\omega,\alpha) \ ~ \mathrm{s.t.} \ ~ \omega\in \mathcal{S}(\alpha),
\end{equation}
where function $F: \mathbb{R}^p \times \mathbb{R}^n \to \mathbb{R}^m$ is a vector-valued jointly continuous function with
$F:=(F_1,F_2,...,F_m)^T$ for the $m$ desired objectives and $\mathcal{A}$ is a nonempty compact subset of $\mathbb{R}^p$. In problem (\ref{P1}), $\mathcal{S}(\alpha)$ is defined as the set of optimal solutions to minimize $f(\omega,\alpha)$ w.r.t. $\omega$, i.e.,
\begin{equation} \label{P2}
 \mathcal{S}(\alpha) = \mathop{\arg\min}_\omega f(\omega,\alpha).
\end{equation}
When $m$ equals 1, problem (\ref{P1}) reduces to the Bi-Level optimization Problem (BLP), which is a widely-used formulation in meta learning, and hence from this perspective, the MOML framework is a generalization of meta learning. In problems (\ref{P1}) and (\ref{P2}), $F$ is called the Upper-Level (UL) subproblem and $f: \mathbb{R}^p \times \mathbb{R}^n \to \mathbb{R}$ is the Lower-Level (LL) subproblem. For meta learning, $F$ contains multiple meta objectives to be achieved for the meta learner and $f$ defines the objective function for current task such as the training loss. In Section \ref{sec_use_cases}, we will see the use of MOML in different meta learning problems, including few-shot learning, NAS, domain adaptation and multi-task learning.
%The UL objective $F$ is a jointly continuous function. In the bi-level problem of this structure, the optimal solutions of the LL problem influence the objective function value on the upper level. Moreover, the optimization variable of the upper level is a parameter for the lower level problem. Specifically, given the UL variable $x$ from $\mathcal{X}$, the LL variable $y$ is an optimal solution of the LL subproblem.

%We can specific the instances of this bi-level problem in both hyperparameter optimization and meta learning problems, which we will discuss next.

For a MOBLP such as problem (\ref{P1}), there are some works \cite{deb2009solving,sinha2011bilevel, ruuska2012constructing} to adopt multi-objective evolutionary algorithms to solve it. However, such solutions have a high complexity without convergence guarantee and cannot be integrated with gradient-based models such as deep neural networks. We are unaware of any gradient-based optimization algorithm with convergence guarantee to solve a MOBLP, which is what we will do in the next section.

\section{Optimization}
\label{sec_optimization}

In this section, we devise a general algorithm to solve the MOBLP (i.e., problem \eqref{P1}) and provide convergence analyses under certain assumptions.%, and the convergence conditions of the gradient-based iterative process.

\subsection{Lower-Level Singleton Condition}
%Before, we discuss the multi-object bi-level optimization problem,
Due to the complicated dependency between UL and LL variables, solving the MOBLP is challenging, especially when optimal solutions of the LL subproblem are not unique.

For a BLP with a single objective in the UL subproblem, many studies \cite{domke2012generic, franceschi2018bilevel, shaban2019truncated} require a Lower-Level Singleton (LLS) condition that the LL subproblem only admits a unique minimizer for every $\alpha\in \mathcal{A}$, which can simplify the optimization process and convergence analyses.

For the MOBLP, the LLS condition is necessary.
If the LLS condition does not hold, the MOBLP is even ill-defined \cite{eichfelder2020twenty}. To see this, suppose for a fixed $\alpha_0$, we get a set of solutions $S(\alpha_0)$ for the LL subproblem. Since $F$ is vector-valued, it is unclear that at which $\omega\in S(\alpha_0)$ the UL subproblem $F$ should be evaluated. %Therefore, for the MOBLP, the LLS condition is necessary.

With the LLS condition, problem \eqref{P1} can be simplified as% two single-level subproblems. Here we formulate a simplified form of problem \eqref{P1} under the LLS condition as
%{\small
\begin{align}
\min_{\alpha\in\mathcal{X}}&\  \varphi(\alpha) = F(\omega^*(\alpha),\alpha)\nonumber\\ \mathrm{s.t.}&\  \omega^*(\alpha) = \mathop{\arg\min}_\omega f(\omega,\alpha).\label{P4}
\end{align}
%}\noindent

%In a more complex situation when the UL and LL subproblem are both MOP. Even all objective functions are strictly convex, this problems may still have an infinite set of efficient solutions. Therefore, the LLS condition can no longer be assumed.

\subsection{Gradient-based Optimization Algorithm}

Here we present a gradient-based optimization algorithm to solve problem \eqref{P4}.

Usually, there is no closed form for the solution $\omega^*(\alpha)$ of the LL subproblem and so it is difficult to optimize the UL subproblem directly. Another approach is to use the optimality condition of the LL subproblem (i.e., $\nabla_\omega  f(\omega,\alpha) = 0$) as equality constraints for the UL subproblem in a way similar to \cite{pedregosa2016hyperparameter}. However, this approach only works for LL subproblems with simple forms and cannot work for general learning models.

Here we use a strategy that replaces the LL subproblem with a dynamical system \cite{franceschi2018bilevel,shaban2019truncated}.
Specifically, we consider the following approximated formulation of problem \eqref{P4} as
\begin{equation}\label{P3}
    \min_{\alpha}\varphi_K(\alpha) = F(\omega_K(\alpha),\alpha),
\end{equation}
where $\omega_K(\alpha)$ denotes an iterative solution of the LL subproblem for a given $\alpha$ and $K$ denotes the number of iterations. With an initialization $\omega_0$ for the LL variable, a sequence $\{\omega_k(\alpha)\}_{k=1}^{K}$ can be generated as
\begin{align}
\omega_{k+1}(\alpha) = \mathcal{T}_k(\omega_k(\alpha),\alpha),\ \forall k=1,\ldots, K-1,\nonumber
\end{align}
where $\mathcal{T}_k$ represents an operator to update $\omega$. Here we consider a  first-ordered gradient descent method for $\mathcal{T}_k$, such as the Stochastic Gradient Descent (SGD) method. Therefore, $\mathcal{T}_k$ can be formulated explicitly as
\begin{align}
\mathcal{T}_k(\omega_k(\alpha),\alpha) = \omega_k(\alpha) - \mu\nabla_\omega f(\omega_k(\alpha),\alpha),\nonumber
\end{align}
where $\mu>0$ denotes the step size and $\nabla_\omega f(\omega_k(\alpha),\alpha)$ denotes the derivative of $f$ w.r.t. $\omega$ at $\omega=\omega_k(\alpha)$.

%Consider a common situation that we need to optimize a number of hyperparameters of the same order as that of parameters in ML. This strategy makes it much simpler to compute the derivative of $F$ than consider the LL subproblem as an equality constraint.

The main advantage of the reformulation in problem (\ref{P3}) is that the UL subproblem becomes an unconstrained MOP.
To solve problem (\ref{P3}), we can adopt any multi-objective optimization algorithm. To make the entire optimization procedure a gradient-based approach, we adopt a simple gradient-based MOP method called Multiple Gradient Descent Algorithm (MGDA) \cite{desideri12}, which shows that the descent direction $d$ for multiple objectives can be found in the convex hull of the gradients of each objective. Specifically, to solve problem (\ref{P3}), MGDA iteratively solves the following quadratic programming problem as
%\begin{smalleralign}[\footnotesize]
\begin{align}
\min_{\bm{\gamma}} &\ \left\|\sum_{i=1}^m\gamma_i \nabla_{\alpha}F_i(\omega_K(\alpha),\alpha) \right\|^2_2 \nonumber\\
\mathrm{s.t.}&\  \gamma_i \geq 0,\ \sum_{i=1}^m \gamma_i=1,\label{eq:mgda}
\end{align}
%\end{smalleralign}
where $\|\cdot\|_2$ denotes the $\ell_2$ norm of a vector and $\gamma_i$ can be viewed as a weight for the $i$th objective. To solve problem (\ref{eq:mgda}), different from \cite{sk18} which uses the the Frank-Wolfe algorithm, we adopt the FISTA algorithm \cite{bt09} with a faster convergence rate. %The derivative of $F_i(\omega_K(\alpha),\alpha)$ (w.r.t. $\alpha$) can be computed by automatic differentiation techniques.
After solving problem (\ref{eq:mgda}), MGDA can update $\alpha$ by minimizing $\sum_{i=1}^m\gamma_iF_i(\omega_K(\alpha),\alpha)$ via SGD. In fact, we can choose any gradient-based MOP method to solve problem (\ref{P3}) and we choose MGDA because of its simplicity and efficiency.

%\begin{figure}[!htb]
%\vskip -0.1in
\renewcommand{\algorithmicrequire}{\textbf{Input:}}
\begin{algorithm}[H]\small
\label{alg1}
\caption{Optimization algorithm for MOML}
\label{alg:MOML}
\begin{algorithmic}[1]
\REQUIRE numbers of iterations ($T$,  $K$), step size ($\mu$, $\nu$)
%\ENSURE Solution of problem \eqref{P4}: $\alpha_N$, $\omega_K^N$
\STATE Randomly initialized $\alpha_0$;
\FOR{$t=1$ {\bfseries to} $T$}
\STATE Initialize $\omega_0^t(\alpha_t)$;
\FOR{$j=1$ {\bfseries to} $K$}
\STATE $\omega_{j}^t(\alpha_t) \gets \omega_{j-1}^t(\alpha_t) - \mu \nabla_\omega f(\omega_{j-1}^t(\alpha_t),\alpha_t)$;
\ENDFOR
\STATE Compute gradients $\nabla_{\alpha}F_i(\omega_{K}^t(\alpha_t),\alpha_t)$ for all the $i$'s;
\STATE Compute weight $\gamma_1,\ldots,\gamma_m$ by solving problem \eqref{eq:mgda};
\STATE $\alpha_{t+1} = \alpha_t - \nu \sum_{i=1}^m\gamma_i\nabla_{\alpha}F_i(\omega_K^t(\alpha_t),\alpha_t)$;
\ENDFOR
\end{algorithmic}
\end{algorithm}
%\end{figure}
\vskip -0.15in

%we need to solve a MOP. In gradient-based MOP, we find a efficient solution by starting from a initialization $x_0\in \mathcal{X}$ and iteratively finding the next solution $x_{n+1}$ which satisfies $F(x_{n+1})\le_p F(x_{n})$. These gradient descent method \cite{fliege2000steepest, tanabe2019proximal} moving the solutions by a descent direction $d$ and step size $\beta$. By constructing a direction $d$ to have positive angles with the gradients of every objective function $F_i$. This iterative process ensures descent can happen in every objective.

The entire algorithm to solve problem (\ref{P4}) is shown in Algorithm \ref{alg:MOML}, which to the best of our knowledge is the first gradient-based optimization algorithm for MOBLPs. In Algorithm \ref{alg:MOML}, we obtain only one solution for MOBLPs, which is different from evolutionary algorithms that can find a population of solutions. How to obtain multiple Pareto-optimal solutions for MOBLPs is beyond the scope of this paper and we will study it in the future work.% In the iterative process of MGDA, we can only find a solution better than the previous one and we do not set any preference. Thus, it may reach any Pareto optimal points, the result may only depend on the initialization point $x_0$.
%We want to use gradient-based methods in the whole process without too much extra computation.

\subsection{Convergence Analysis}

In this section, we analyze the convergence of Algorithm \ref{alg:MOML}.

In the following, we introduction some basic assumptions which are widely adopted in BLPs. Firstly, we make the assumption about the LLS condition mentioned earlier.
\begin{assumption} \label{Ass1}
%$\mathcal{A}$ is a nonempty compact set, $F$ is jointly continuous function.
$\mathop{\arg\min}_\omega f(\omega,\alpha)$ is a singleton for every $\alpha\in \mathcal{A}$, and $\{\omega_K(\alpha)\}$ is uniformly bounded on $\mathcal{A}$.
\end{assumption}

Under this assumption, we can get the following result.
\begin{theorem} \label{theorem1}
 If Assumption \ref{Ass1} is satisfied, the vector-valued function $F(\omega^*(\alpha),\alpha)$ is continuous w.r.t. $\alpha$.
\end{theorem}

Because $\mathcal{A}$ is a compact set, Theorem \ref{theorem1} implies the existence of solutions. Theorem \ref{theorem1} and the uniform convergence of $\omega_K(\alpha)$ can further imply the convergence for the solution of the LL subproblem. This result is similar to that of the BLP problem \cite{franceschi2018bilevel}.

For the convergence of the UL subproblem, we need to analyze sets of minimal points of the images of perturbed function $\varphi_K(\alpha)$ and $\varphi(\alpha)$. The convexity of those vector-valued functions can be defined as the P-convex. Moreover, we consider the most natural set convergence under this setting, i.e., the Kuratowski-Painlev\'{e} set-convergence. Please refer to those definitions in Appendix \ref{A}.

In this following, We make another assumption.
\begin{assumption} \label{Ass2}
It is assumed that
\begin{itemize}[noitemsep,topsep=0pt,parsep=0pt,partopsep=0pt]
    \item $F(\cdot,\alpha)$ is uniformly Lipschitz continuous;
    \item The iterative sequence $\{\omega_k(\alpha)\}_{k=1}^K$ converges uniformly to $\omega^*(\alpha)$ on $\mathcal{X}$ as $K \to +\infty$;
    \item $\mathcal{A}$ is a convex set;
    \item $\varphi_K$ is P-convex and $\varphi$ is strictly P-convex.
    %\item For any $a\in \mathbb{R}^p$, if the solution set $\{x: \varphi(x) \le_p a\} $ is not empty, then it is bounded.
\end{itemize}
\end{assumption}
Note that the first two items in Assumption \ref{Ass2} are widely used to analyze BLPs \cite{franceschi2018bilevel}, and the last two are adopted in the stability analysis of MOPs \cite{lucchetti2004stability}.
%Since the Different from the case where the upper level is a scale-valued function, we need to show the convergence of the set of efficient solutions.

Then we have the following convergence results. % the result below from general facts on the stability of convex vector optimization problems.

\begin{theorem}\label{theorem2}
Under Assumptions \ref{Ass1} and \ref{Ass2}, the Kuratowski-Painlev\'{e} set-convergence of both the minimal point set and efficient solution set in Algorithm \ref{alg:MOML} holds, i.e.,
\begin{displaymath}
\mathrm{Min}\ \varphi_K(\alpha) \to \mathrm{Min}\ \varphi(\alpha),\
\mathrm{Eff}\ \varphi_K(\alpha) \to \mathrm{Eff}\ \varphi(\alpha).
\end{displaymath}
%For every $x_K$ such that $x_K\in \mathrm{Eff}( \varphi_K(x))$, then for any limit point $\bar{x}$ of the sequence $\{x_K\}$, we have $\bar{x}\in \mathrm{Eff}( \varphi(x))$.
\end{theorem}

Theorem \ref{theorem2} shows that Algorithm \ref{alg:MOML} to solve MOBLPs can converge under Assumptions \ref{Ass1} and \ref{Ass2}.
%The convergence properties of MOML rely on gradient-based method for both UL and LL subproblems.  Specifically, how to find such $y_k$ converge to $y$ uniformly and get the efficient solution $x_k\in \mathrm{Eff}( \varphi_K(x))$.
To satisfy the LLS condition, the LL subproblem could be strongly convex and then $\omega_K(\alpha)$ can converge uniformly to $\omega(\alpha)$ at a linear rate. For the UL subproblem, it has been shown by \cite{desideri12} that when MGDA converges, it can reach a Pareto-stationary point. When the vector-valued function in the UL subproblem is strictly P-convex, it can converge to a Pareto-optimal solution \cite{tanabe2019proximal}.

%Once we separate the original bi-level problem into two single-level subproblems. We use a convergent sequence $\{y_K\}$ to represent the solution of LL objective and absorb it into upper-level, since it uniformly converge to $y(x)$. Therefore, the UL problem turns into a unconstrained multiobjective optimization problems $F(x,y_K(x))$. Therefore, if the whole estimate approach is suitable and convergent, the final convergence depends on the convergence of multi-objective optimization on the upper-level, in other words, how to get $x_K\in \mathrm{Eff}( \varphi_K(x))$.

\section{Use Cases of MOML}
\label{sec_use_cases}

In this section, we introduce several use cases of the MOML framework, including few-shot learning, NAS, domain adaptation, and multi-task learning.

\subsection{Few-Shot Learning}

Few-Shot Learning (FSL) aims to tackle the problem of training a model with only a few training samples \cite{wang2020generalizing}. Recently, FSL is widely studied from the perspective of meta learning by using the prior knowledge in the meta training process. Most studies in FSL only consider the classification performance. %Here we show this problem should consider more objectives and can be solved by our proposed MOML framework.
However, in real world applications, the performance is not the only focus. For example, we expect FSL models to not only have good performance but also be robust to adversarial attacks \cite{pgd17}, which may improve the generalization of FSL models. In the following, we can see that this setting can naturally be modeled by the proposed MOML framework.

\subsubsection{Problem Formulation}

% Most studies in FSL only consider the classification performance. %Here we show this problem should consider more objectives and can be solved by our proposed MOML framework.
% However, in real world applications, the performance is not the only focus. For example, we expect FSL models to not only have good performance but also be robust to adversarial attacks \cite{pgd17}, which may improve the generalization of FSL models. In the following, we can see that this setting can naturally be modeled by the proposed MOML framework. %if we consider more complicated situation, for example, we expect the model adapting with the optimized meta parameter $\alpha$ can maintain both robustness and accuracy on the novel tasks. Therefore, we can integrate our MOML framework to FSL problem.

Suppose there are a base dataset $\mathcal{D}_{base}$ with a category set $\mathcal{C}_{base}$ and a novel dataset $\mathcal{D}_{novel}$ with a category set $\mathcal{C}_{novel}$, where $\mathcal{C}_{base} \cap \mathcal{C}_{novel}=\emptyset$. The goal of FSL is to adapt the knowledge learned from $\mathcal{D}_{base}$ to help the learning for $\mathcal{D}_{novel}$. % with only $k$ labeled samples per class.
In the $i$th meta training episode, we generate from $\mathcal{D}_{base}$ a $N$-way $k$-shot classification task, which consists of a support set $\mathcal{D}_{base}^{s(i)}$ and a query set $\mathcal{D}_{base}^{q(i)}$. %Specifically, we first randomly sample $N$ classes from $\mathcal{C}_{base}$ and then sample $k$ and $q$ labeled data for each sampled class to constitute a support set $\mathcal{D}_{base}^{s(i)}$ and a query set $\mathcal{D}_{base}^{q(i)}$, respectively. Thus the support set and the query set can be considered as a training set with $N\times k$ labeled data and a test set with $N\times q$ labeled data, respectively.
For the robustness, we add perturbations generated by the Projected Gradient Descent (PGD) method \cite{pgd17} into each data point from $\mathcal{D}_{base}^{q(i)}$ to generate a perturbed query set $\mathcal{D}_{base}^{q(i), adv}$.
The objective function of the FSL model that considers both the performance and the robustness can be formulated as
%{\small
\begin{smalleralign}
\min_{\alpha}~&\big(\mathcal{L}_{F}(\omega^{*(i)}(\alpha), \alpha, \mathcal{D}_{base}^{q(i)}),\mathcal{L}_{F}(\omega^{*(i)}(\alpha), \alpha, \mathcal{D}_{base}^{q(i), adv})\big)\nonumber \\
\mathrm{s.t.}~&\omega^{*(i)}(\alpha)=\arg\min_{\omega}\mathcal{L}_{F}(\omega, \alpha, \mathcal{D}_{base}^{s(i)}),
\label{eq:fsl}
\end{smalleralign}
%}\noindent
where $\omega$ represents model parameters, $\alpha$ denotes the meta parameters to encode common knowledge that can be transferred to novel tasks, and  $\mathcal{L}_{F}(\omega,\alpha,\mathcal{D})$ denotes the average classification loss of a model with model parameters $\alpha$ and meta parameters $\omega$ on a dataset $\mathcal{D}$. In the UL subproblem of problem (\ref{eq:fsl}), the first objective  measures the classification loss on the query set based on $\omega^{*(i)}(\alpha)$ obtained by solving the LL subproblem and the second objective measures the robustness via the classification performance on the perturbed query set. Problem (\ref{eq:fsl}) provides a general formulation, which depends on what $\alpha$ represents, for FSL. To see this, by taking MAML as an example, $\alpha$ can represent the initialization of model parameters shared among tasks and $\omega$ denotes task-specific model parameters. It is easy to see that problem (\ref{eq:fsl}) fits the MOML framework and we can use Algorithm \ref{alg:MOML} to solve it.
% such as the initialization of model parameters in MAML,

%into the UL subproblem in Problem (\ref{eq:fsl}), so the UL objective becomes a vector-valued function and this multi-objective FSL problem can be solved by the MOML framework. Here we use a representative meta learning algorithm, MAML \cite{finn2017model} as a example and compare it with our MOML framework as follow.

\subsubsection{Experiments}
%\subsubsection{Experimental Settings}
Experiments are conducted on two FSL benchmark datasets, CUB-200-2011 (referred to as CUB) \cite{wah2011caltech} and \textit{mini}-ImageNet \cite{vinyals2016matching}. %The CUB dataset contains 200 classes and 11,788 images in total. Following \cite{hilliard2018few}, we randomly split this dataset into a base dataset containing 100 classes, a validation dataset containing 50 classes, a novel dataset containing the rest 50 classes. The \textit{mini}-ImageNet dataset contains 100 classes with 600 images per class, sampling from the ImageNet dataset \cite{deng2009imagenet}. Following \cite{ravi2016optimization}, this dataset is splited into 64, 16 and 20 classes for the base, validation and novel datasets, respectively.
%A four-layer convolution network (Conv-4) is used as the backbone with an input size of $84\times 84$ for both MAML and MOML methods. In the meta training, we train randomly sample 16 instances per class as the query set in each episode. The adversarial attack on the query set is performed by the PGD attack with a perturbation size $\epsilon =2/255$ and it takes $7$ iterative steps with the step size of $2.5\epsilon$. In the meta testing, we generate 600 $N$-way $k$-shot tasks from $\mathcal{D}_{novel}$, where each task has 16 samples for testing. We compare with MAML as problem (\ref{eq:fsl}) can reduce to MAML when there is only the first objective in the UL subproblem.
Experimental settings are put in Appendix \ref{sec:setting_FSL}.

%\subsubsection{Experimental Results}

As presented in Table \ref{tab:result_FSL}, the average results over 600 testing tasks in terms of the clean classification accuracy and the PGD accuracy show that MOML can find a trade-off solution compared with MAML which only focus on the classification accuracy. Though the classification accuracy of the MOML is slightly lower than that of MAML by around 5\%, the robustness of MOML is greatly improved compared with the MAML (i.e., up to about 8.4 times). %However, these result is not enough to show whether MOML reaches the set of minimum points of the vector-valued target function $(\mathcal{L}, \mathcal{L}_{adv})$.

\begin{table}[!htb]
\vskip -0.2in
\caption{Results of MOML and MAML on two datasets under the PGD attack.}
\resizebox{\linewidth}{!}{
\begin{tabular}{c|c|ccccc}
\toprule
\textbf{Dataset} & \textbf{Setting} & \textbf{Model} &\textbf{Clean Acc.}  & \textbf{PGD Acc.}\\
\midrule
\multirow{4}*{\rotatebox{90}{\textbf{CUB}}} & \multirow{2}*{1-shot}
                            &MAML  &54.41$\pm$0.96  &4.08$\pm$0.41 \\
 ~  & ~                     &MOML  &47.74$\pm$0.76  &25.67$\pm$0.65 \\
\cmidrule{2-5}
~  & \multirow{2}*{5-shot}  &MAML  &76.12$\pm$0.71  &8.95$\pm$0.45 \\
 ~  & ~                     &MOML  &72.97$\pm$0.81  &43.41$\pm$0.92 \\
\midrule
\multirow{4}*{\rotatebox{90}{\textbf{{\scriptsize\textit{mini}-ImageNet}}}} & \multirow{2}*{1-shot}                                   &MAML  &46.58$\pm$0.80  &3.24$\pm$0.24 \\
~ & ~                       &MOML  &40.03$\pm$0.81  &26.14$\pm$0.92 \\
\cmidrule{2-5}
~  & \multirow{2}*{5-shot}    &MAML  &62.85$\pm$0.76  &4.70$\pm$0.28 \\
~ & ~             &MOML  &57.06$\pm$0.74  &38.95$\pm$0.73 \\
\bottomrule
\end{tabular}
}
% \vskip 0.05in
\label{tab:result_FSL}
%\vskip -0.15in
\end{table}

\begin{figure} [!htb]
\vskip -0.1in
 \centering
  \subfigure[5-way 1-shot]{\includegraphics[width=0.48\linewidth]{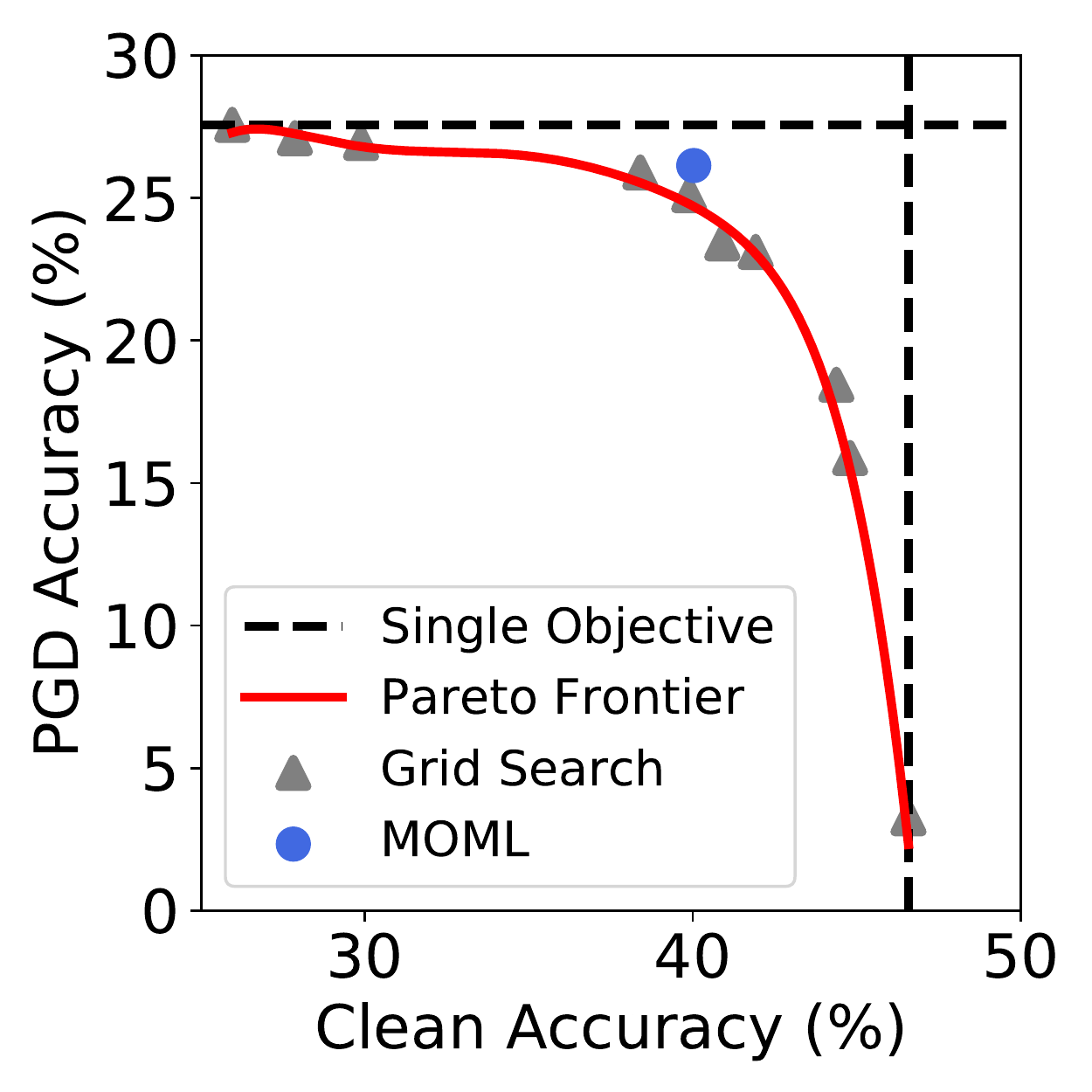}{\label{fig:fig1a}}}
  \subfigure[5-way 5-shot]{\includegraphics[width=0.48\linewidth]{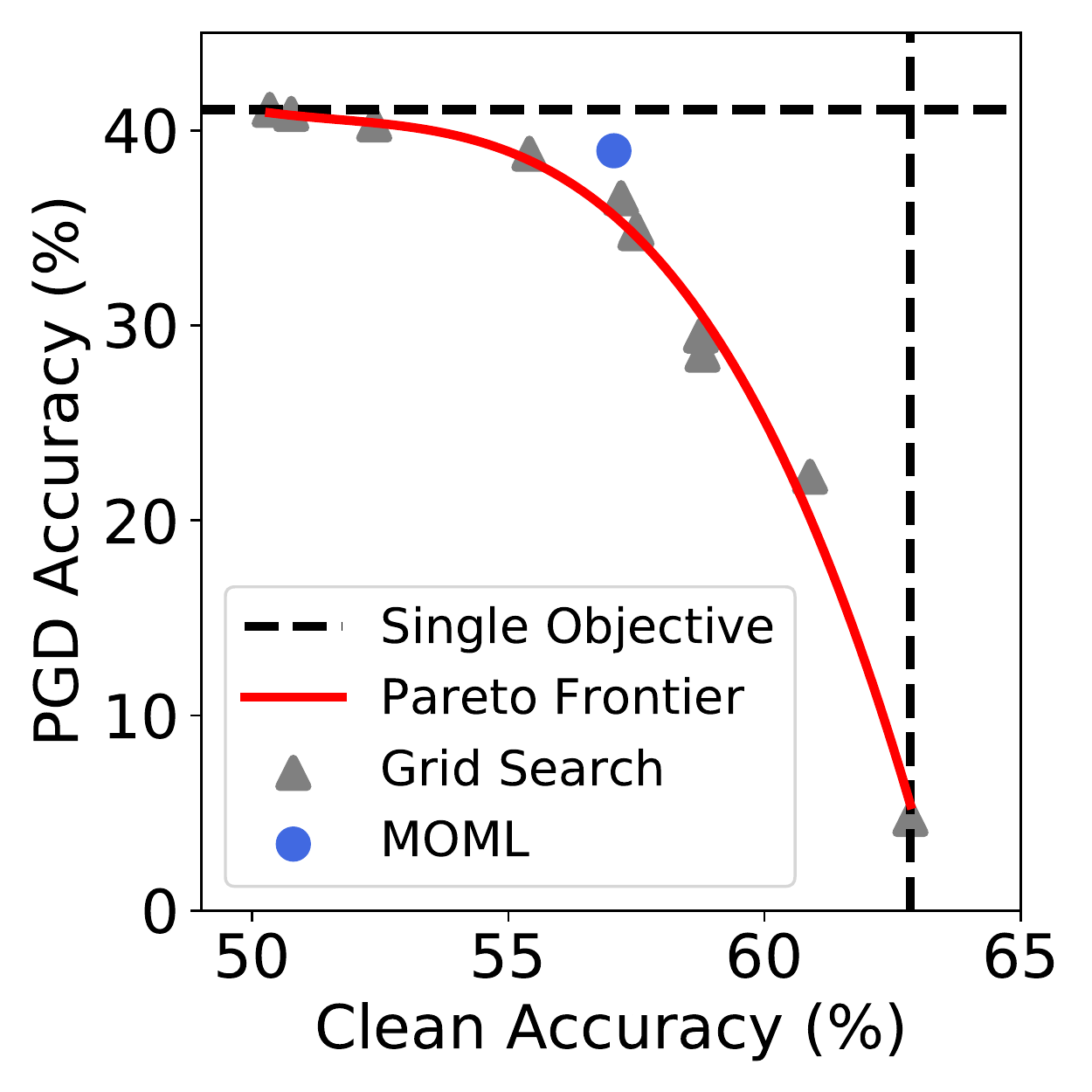}{\label{fig:fig1b}}}
  \vskip -0.1in
  \caption{The grid search results marked in grey triangles and the solution of MOML marked in the blue circle on the \textit{mini}-ImageNet dataset under two settings. The red line represents the approximation of the Pareto frontier.}
\label{fig1}
\vskip -0.2in
\end{figure}

To check whether the MOML can identify a nearly Pareto-optimal solution, we replace the UL subproblem of MOML with a convex sum of all the $m$ objectives via grid search on combination coefficients and then we can find the approximate minimum points set, which is also called the Pareto frontier. According to Figure \ref{fig1}, we can see that the solution of MOML falls very well on the Pareto frontier, which demonstrates the effectiveness of MOML.

\subsection{Neural Architecture Search}
NAS aims to design the architecture of neural networks in an automated way. Most NAS methods focus on searching architectures with the best accuracy. However, in real-world applications, other important factors, such as the network size and robustness, should be considered. To achieve this, we propose a multi-objective NAS method based on the MOML framework.

\subsubsection{Problem Formulation}
% We integrate our MOML framework to a gradient-based NAS method called DARTS \cite{lsy19}. DARTS uses a softmax function to relax the discrete NAS search space.
% \begin{equation}
% \label{eq:darts_softmax}
%     \bar{o}_{i,j}(x_i) = \sum_{k=1}^{K}\frac{exp(\alpha^{(i,j)}_o)}{\sum_{o^{\prime}\in\mathcal{O}}exp(\alpha^{(i,j)}_{o^{\prime}})}o^k(x_i)
% \end{equation}
% where $o(x)$ is an operation function from the search space $\mathcal{O}$ with probability $\alpha^{(i,j)}_o$ between a pair of nodes $(i,j)$. $K$ is the total number of operation types in the search space. After relaxation, the search space of DARTS is continuous and differentiable. When the search procedure finishes, the final architecture can be determined by the operation with the largest probability in each cell.

By following the DARTS method \cite{lsy19}, in an operation space denoted by $\mathcal{O}$, each element is an operation function $o(\cdot)$ and %The goal of DARTS is to search for a cell, which can be stacked to form a neural network architecture.
each cell is a  directed acyclic graph with $N$ nodes, where each node represents a hidden representation and each edge $(i,j)$ denotes a candidate operation $o(\cdot)$ with a probability $\alpha_o^{(i,j)}$.
Therefore, $\bm{\alpha}=\{\alpha_o^{(i,j)}\}_{(i,j)\in\bm{E},o\in\mathcal{O}}$ is a representation of the neural architecture, where $\bm{E}$ denotes the set of all the edges in all the cells. The entire dataset is split into a training dataset denoted by $\mathcal{D}_{tr}$ and a validation dataset denoted by $\mathcal{D}_{val}$.
%DARTS formulates NAS as a bi-level optimization problem \cite{lsy19}
%\begin{equation}
%\begin{aligned}
%\min_{\bm{\alpha}} \quad & \mathcal{L}({\omega}^*(\bm{\alpha}),\bm{\alpha}, \mathcal{D}_{val})\\
%\mathrm{s.t.} \quad & {\omega}^*(\bm{\alpha}) = \arg\min_{{\omega}}~\mathcal{L}({\omega},\bm{\alpha}, \mathcal{D}_{tr}),
%\end{aligned}
%\label{eq:darts_loss}
%\end{equation}

%where $\mathcal{L}(\omega,\bm{\alpha},\mathcal{D})=\frac{1}{|\mathcal{D}|}\sum_{(\mathbf{x},y)\in\mathcal{D}}~\ell({\omega},\mathbf{x},y)$ represents the total loss w.r.t. dataset $\mathcal{D}$ when giving a architecture $\bm{\alpha}$ and $\ell(\omega, \mathbf{x}, y)$ denotes the loss function for each sample.
% In this bi-level setting, $\min_{{\alpha}} \mathcal{L}_{val}({\omega}^*({\alpha}),{\alpha})$ is called the upper-level subproblem and $\min_{{\omega}}\mathcal{L}_{tr}({\omega},{\alpha})$ is called the lower-level subproblem.
% In lower-level subproblem, the objective is to minimize the training loss for architecture $\alpha$ with weight $\omega$. We aim to minimize the training loss of clean samples as
% \begin{equation}
% \mathcal{L}_{tr}({\omega},{\alpha})=\frac{1}{|\mathcal{D}_{tr}|}\sum_{(\mathbf{x},y)\in\mathcal{D}_{tr}}~\ell({\omega},\mathbf{x},y)
% \end{equation}

The multi-objective NAS considers three objectives: the classification accuracy, adversarial robustness and the number of parameters, and we formulate three corresponding losses as $\mathcal{L}_{N}(\omega, \bm{\alpha}, \mathcal{D}_{val})$, $\mathcal{L}_{N}(\omega, \bm{\alpha}, \mathcal{D}_{val}^{adv})$, and $\mathcal{L}_{nop}(\bm{\alpha})$, where ${\omega}$ denotes all the model parameters in the neural network, $\mathcal{L}_{N}(\omega, \bm{\alpha}, \mathcal{D})$ denotes the average classification loss on a dataset $\mathcal{D}$ of a neural network with parameters ${\omega}$ and an architecture $\bm{\alpha}$, and $\mathcal{D}_{val}^{adv}$ denotes the perturbed validation dataset by adding perturbations on each data point. %Problem (\ref{eq:darts_loss}). $\mathcal{L}_{adv}(\omega, \bm{\alpha}, \mathcal{D}_{val}^{adv})$ represents the adversarial robustness of the searched architecture after training and formulated as
%\begin{equation}
%    \mathcal{L}_{adv}(\omega,\bm{\alpha},\mathcal{D}_{val}^{adv})=\frac{1}{|\mathcal{D}_{val}|}\sum_{(\mathbf{x},y)\in\mathcal{D}_{val}}~\ell({\omega},\mathbf{x}+\delta,y),
%\end{equation}
%where $\delta$ is the specific perturbation for data $\mathbf{x}$.
% We can divide these goals into two categories. The first group $\{F_{nop}\}$ has an objective that depends on the architecture. The second group $\{F_{acc}, F_{adv}\}$ consists of objectives that depend on both the architecture and the weights.
To formulate $\mathcal{L}_{nop}(\bm{\alpha})$, we denote by $n_{o}$ the number of parameters associated with an operation $o$ and by $N_{nop}(\bm{\alpha})$ the number of parameters in a searched architecture $\bm{\alpha}$. Then $N_{nop}(\bm{\alpha})$ can be computed by $N_{nop}(\bm{\alpha}) = \sum_{(i,j)\in\bm{E}} n^{(i,j)}$,
% \begin{equation}
%     N_{nop}({\alpha}) = \sum_{(i,j)\in\mathbf{E}} n_o^{(i,j)},
% \label{eq:num_param}
% \end{equation}
where $n^{(i,j)}$ is the number of parameters of the searched operation on the edge $(i,j)$. As we determine the operation of each edge by selecting the one with the largest probability, hence we have $n^{(i,j)} = n_{{\arg\max}_{o\in\mathcal{O}}~\alpha^{(i,j)}_o}$. As the $\arg\max$ operation is non-differentiable, we use the softmax function to approximate it and hence $N_{nop}(\bm{\alpha})$ can be approximated as
%\begin{smalleralign}
$\hat{N}_{nop}(\bm{\alpha}) = \sum_{(i,j)\in\bm{E}} \sum_{o\in\mathcal{O}}\frac{\exp(\alpha^{(i,j)}_o)}{\sum_{o^{\prime}\in\mathcal{O}}\exp(\alpha^{(i,j)}_{o^{\prime}})} n_o$. %\nonumber
%\label{eq:num_param_soft}
%\end{smalleralign}
To search a network architecture with an expected size $L$, $\mathcal{L}_{nop}(\bm{\alpha})$ can be formulated as
$\mathcal{L}_{nop}(\bm{\alpha}) = |\hat{N}_{nop}(\bm{\alpha})- L|$.

Therefore, the overall formulation for the multi-objective NAS is formulated as
%{\small
\begin{smalleralign}
\min_{\bm{\alpha}}~& (\mathcal{L}_{N}({\omega}^*(\bm{\alpha}),\bm{\alpha}, \mathcal{D}_{val}), \mathcal{L}_{N}({\omega}^*(\bm{\alpha}),\bm{\alpha}, \mathcal{D}_{val}^{adv}), \mathcal{L}_{nop}(\bm{\alpha}))\nonumber\\
\mathrm{s.t.}~& {\omega}^*(\bm{\alpha}) = \arg\min_{{\omega}}~\mathcal{L}_N({\omega},\bm{\alpha}, \mathcal{D}_{tr}).\label{eq:nas_loss}
\end{smalleralign}
%}\noindent
Obviously problem (\ref{eq:nas_loss}) matches the MOML framework. It is easy to see that the DARTS method is a special case of problem (\ref{eq:nas_loss}) when its UL subproblem contains the first objective only and hence problem (\ref{eq:nas_loss}) generalizes the DARTS method by considering two more factors.
Compared with the NSGANetV2 method \cite{lu2020nsganetv2} which is based on a multi-objective bi-level evolutionary algorithm, MOML is more efficient and has convergence guarantee. Moreover, NSGANetV2 uses ensembled surrogate models to predict the accuracy of an architecture, which may incur a performance gap between the UL and LL subproblems. In the LL subproblem of NSGANetV2, it only chooses over 300 candidate architectures for evaluation with a supernet constructed for weight sharing, which may lead to suboptimal solutions. %On the other hand, MOML follows the standard training procedure since only one optimal candidate is searched for evaluation.

% In the second group, $F_{acc}$ is the classification accuracy of an architecture after training. We aim to minimize the validation loss of clean samples as
% \begin{equation}
% F_{acc}({\omega},{\alpha}) = \frac{1}{|\mathcal{D}_{val}|} \sum_{(\mathbf{x},y)\in\mathcal{D}_{val}}\ell({\omega},\mathbf{x},y).
% \label{eq:acc_train}
% \end{equation}
% $F_{adv}$ is the adversarial robustness of an architecture after training. We aim to minimize the validation loss of the perturbed samples as
% \begin{equation}
% F_{adv}({\omega},{\alpha}) = \frac{1}{|\mathcal{D}_{val}|} \sum_{(\mathbf{x},y)\in\mathcal{D}_{val}}\ell({\omega},\mathbf{x'},y).
% \label{eq:adv_train}
% \end{equation}

\subsubsection{Experiments}
%\subsubsection{Experimental Settings}
%The search space and training procedure of MOML adopt the same settings as DARTS. Specifically, in both normal and reduction cells, the set of operations $\mathcal{O}$ contains eight operations, including $3\times3$ separable convolutions, $5\times5$ separable convolutions, $3\times3$ dilated separable convolutions, $5\times5$ dilated separable convolutions, $3\times3$ max pooling, $3\times3$ average pooling, identity, and zero. Half of the training set is used for training a model, and the other half for validation. A small network of 8 cells is trained with a batch size 64 and initial channels 16 for 50 epochs. The SGD optimizer with the momentum $0.9$ and the weight decay $3\times 10^{-4}$ are used.

%In the evaluation stage, a large network of 20 cells is trained on the full training set for 600 epochs with the batch size as 96, the initial number of channels as 36, the length of a cutout as 16, the dropout probability as 0.2, and auxiliary towers of weight as 0.4. The accuracy is tested on the full testing set. Adversarial examples are generated using the PGD attack with the perturbation size $\epsilon = 1/255$ and the PGD attack takes 10 iterative steps with the step size of $2.5\epsilon$ as suggested in \cite{pgd17}.

%\subsubsection{Experimental Results}
In Table \ref{tab:result_cifar10}, we compare the proposed method with the DARTS method on the CIFAR-10 dataset \cite{krizhevsky2009learning}.
% Compared to the original DARTS in \cite{lsy19}, our MOML method significantly improves the robustness of all models under the same model size while the classification accuracy on clean examples keeps almost the same.
We search for neural networks with different expected sizes (i.e., different $L$'s) via the MOML method. To make the network size searched by DARTS comparable with that of MOML under different settings, we use different numbers of initial channels in DARTS during the evaluation process. Experimental settings are put in Appendix \ref{sec:setting_NAS} and the experimental results of ablation study are put in Appendix \ref{app:nas}.

\begin{table}[!htb]
\vskip -0.1in
\centering
\caption{Comparison between MOML and DARTS on the CIFAR-10 dataset. $\uparrow$ indicates that a larger value is better, while $\downarrow$ implies that a lower value is better. ``$\{\text{DARTS-C}\#\text{channels}\}$" means that the architecture searched by DARTS is evaluated with the initial number of channels as ``channels''. ``$\{\text{MOML-V}\#\text{size}\}$" denotes the architecture searched by MOML with $L$ as ``size''.}
\begin{tabular}{lccccc}
\toprule
\multirow{2}*{\textbf{Architecture}} & \textbf{Params} & \textbf{Clean Acc.}  & \textbf{PGD Acc.}\\
~ & \textbf{(MB)}~$\downarrow$ & \textbf{(\%)}~$\uparrow$ & \textbf{(\%)}~$\uparrow$ \\
\midrule
DARTS-C26    &1.787  &96.91  &28.45 \\
DARTS-C30    &2.354  &97.13  &31.53 \\
DARTS-C34    &2.998  &97.34  &30.31 \\
\midrule
MOML-V1    &1.754  &96.48  &42.66 \\
MOML-V2    &2.367  &97.18  &36.15 \\
MOML-V3    &3.018  &97.25  &35.22 \\
\bottomrule
\end{tabular}
\label{tab:result_cifar10}
\vskip -0.35in
\end{table}

% According to the results shown in Table \ref{tab:result_cifar10}, we can see that MOML remarkably improves the robustness with comparable classification accuracy.
Compared with DARTS, the MOML method with a comparable size improves the robustness and has comparable and even slightly better performance on clean examples. For example, compared MOML-V1 with DARTS-C26, the PGD accuracy increases by about 14\%, while the clean accuracy only drops around 0.5\%.
So experimental results in Table \ref{tab:result_cifar10} show that the MOML method can search more robust architectures with similar model size and comparable classification accuracy than the DARTS method.

\subsection{Semi-Supervised Domain Adaptation}

Semi-Supervised Domain Adaptation (SSDA) aims to address the domain shift between two domains so that the model trained in a label-rich source domain can be adapted to a target domain with limited labeled samples and abundant unlabeled samples \cite{yzdp20}. A widely-used approach for SSDA is to align the distributions of two domains via some measures on the domain discrepancy. There are usually three objectives to be considered, including two training losses on two domains and an alignment loss to measure the domain discrepancy. While existing works such as \cite{yao2015semi, saito2019semi, zhu2020deep} optimize all the objectives by simply computing a weighted sum of them, we formulate the SSDA problem as a multi-objective bi-level problem under the MOML framework.

\subsubsection{Problem Formulation}
Given a source domain $\mathcal{S}$ and a target domain $\mathcal{T}$, the source domain has a large labeled dataset $\mathcal{D}_{\mathcal{S}}$  %=\{(\mathbf{x}^{\mathcal{S}}_i,y^{\mathcal{S}}_i)\}_{i=1}^{n_{\mathcal{S}}}$, where $n_{\mathcal{S}}$ denotes the number of data points in the source domain.
and the target domain has a limited labeled dataset   $\mathcal{D}^l_{\mathcal{T}}$ %=\{(\mathbf{x}^{\mathcal{T}}_i,y^{\mathcal{T}}_i)\}_{i=1}^{n_{\mathcal{T}}}$, where $n_{\mathcal{T}}$ denotes the number of labeled data points in the target domain and $n_{\mathcal{S}}\gg n_{\mathcal{T}}$,
as well as a large unlabeled target dataset $\mathcal{D}^u_{\mathcal{T}}$, where $\mathcal{D}_{\mathcal{T}}=\mathcal{D}^l_{\mathcal{T}}\bigcup\mathcal{D}^u_{\mathcal{T}}$ denotes the entire dataset for the target domain. %=\{(\mathbf{x}^{\mathcal{T}}_i)\}_{i=1}^{n_{\mathcal{T},u}}$.  $\mathcal{D}_{\mathcal{S}}$ and  $\mathcal{D}_{\mathcal{T}}$ are sampled from different data distributions \textit{p} and \textit{q}, respectively, and $\textit{p}\not = \textit{q}$.
The learning model consists of a feature extractor parameterized by $\omega$ and a classifier parameterized by $\psi$. The average classification losses in the source and target domains are represented by $\mathcal{L}_{D}(\omega,\psi,\mathcal{D}_{\mathcal{S}})$ and $\mathcal{L}_{D}(\omega,\psi,\mathcal{D}^l_{\mathcal{T}})$, respectively. The alignment loss denoted by $\mathcal{L}_{M}(\omega, \alpha, \mathcal{D}_{\mathcal{S}}, \mathcal{D}_{\mathcal{T}}^{u})$ is measured by the Maximum Mean Discrepancy (MMD) \cite{gretton2012optimal}, where $\alpha$ is the initialization of $\omega$.
Then we can formulate the SSDA problem under the MOML framework as
\begin{smalleralign}
\min_{\alpha,\psi}~&(\mathcal{L}_{D}(\omega^*,\psi,\mathcal{D}_{\mathcal{S}}), \mathcal{L}_{D}(\omega^*,\psi,\mathcal{D}^l_{\mathcal{T}}), \mathcal{L}_{M}(\omega^*, \alpha, \mathcal{D}_{\mathcal{S}}, \mathcal{D}_{\mathcal{T}}^u)) \nonumber \\
\mathrm{s.t.}~& \omega^*=\arg\min_{\omega}\mathcal{L}_{M}(\omega, \alpha, \mathcal{D}_{\mathcal{S}}, \mathcal{D}_{\mathcal{T}}^{u}),
\label{eq:da_loss}
\end{smalleralign}
where $\omega^*$ in the LL subproblem relies on $\alpha$ and we omit such dependency for the notational simplicity. In the LL subproblem, we aim to learn a feature extractor to align the data distributions in two domains by optimizing $\omega$ with an initialization $\alpha$ and in the UL subproblem, we expect to improve the feature extractor further by updating $\alpha$ in the alignment loss and learn the classifier via minimizing the two classification losses. $\alpha$ acts similar to the parameter initialization in MAML (i.e., $\alpha$ in problem (\ref{eq:fsl})) and helps learn $\omega$ in the LL subproblem. %is finally updated to result in performing well on both two classification losses and also the alignment loss by combining their gradients in the UL subproblem.
Different from \cite{li2018learning, guo2020learning} that consider a single-objective problem by linearly combining multiple objectives, we cast these three objectives as a multi-objective problem in the UL subproblem of MOML. Compared with MAML, problem (\ref{eq:da_loss}) learns both $\psi$ and $\alpha$ simultaneously, but does not require any adaptation on the testing process.

\subsubsection{Experiments}
%\subsubsection{Experimental Settings}

Experiments are conducted on the Office-31 dataset \cite{saenko2010adapting}, which has 3 domains: Amazon (\textbf{A}), Webcam (\textbf{W}) and DSLR (\textbf{D}). By following \cite{tzeng2014deep,long2017deep}, we construct all six transfer tasks. %: A$\rightarrow$D, D$\rightarrow$A, A$\rightarrow$W, W$\rightarrow$A, D$\rightarrow$W, and W$\rightarrow$D.
Baseline models in comparison include a deep neural network (denoted by `S+T') that is trained on $\mathcal{D}_{\mathcal{S}}\bigcup\mathcal{D}^l_{\mathcal{T}}$ and the DSAN method \cite{zhu2020deep} that is trained on $\mathcal{D}_{\mathcal{S}}\bigcup\mathcal{D}_{\mathcal{T}}$. We also compare with a variant of the MOML method denoted by MOML$_\emph{w/o MGDA}$ which combines multiple objectives in the UL subproblem as a single objective by tuning the combination weights.
Experimental settings are put in Appendix \ref{sec:setting_SSDA}.

%We use the ResNet-50  model \cite{he2016deep} pretrained on the ImgeNet dataset \cite{russakovsky2015imagenet} as the backbone network followed by a Fully-Connected (FC) layer with the tanh activation function. The same network structure is used for all baseline methods.
%For all experiments, the mini-batch SGD with a Nestorov momentum of 0.9 is used for optimization. The batch size is set to 96, including 32 images in the source, labeled target, and unlabeled target domains, respectively. The learning rate of the classification network is set to 0.01, which is 10 times that of the backbone network.
%To suppress the noisy activation at the early stages of training, instead of fixing the adaption factor $\lambda$, we gradually change it from 0 to 1 according to epochs as: $\lambda = 2/\exp(-10 \times \mathrm{epochs})-1$.

%\subsubsection{Experimental Results}

\begin{table}[!htb]
\vskip -0.15in
\caption{Accuracy (\%) on the Office-31 dataset for semi-supervised domain adaptation.}
\centering
\resizebox{\linewidth}{!}{
\begin{tabular}{@{}lccccccc@{}}
\toprule
Method & A$\rightarrow$D & D$\rightarrow$A & A$\rightarrow$W & W$\rightarrow$A & D$\rightarrow$W & W$\rightarrow$D &Avg  \\

\midrule
S+T&  93.58&	74.16&	92.17&	74.08&	98.01&	\textbf{100}&	88.67   \\
DSAN& 93.83&	76.82&	93.59&	75.68&	\textbf{98.43}&	\textbf{100}&	89.73   \\
MOML$_\emph{w/o MGDA}$& \textbf{94.32}&	76.91&	94.16&	\textbf{75.99}&	 97.72&	\textbf{100}&	89.85\\
\textbf{MOML}& 94.08&	\textbf{77.13}&	\textbf{94.59}&	75.96&	98.36&	\textbf{100}&	\textbf{90.02} \\
\bottomrule
\end{tabular}
}
%\vskip -0.1in
\label{office31-table}
\vskip -0.15in
\end{table}

In each experiment, we randomly selected three labeled samples per class in the target domain for training and the remaining samples are for the unlabeled target dataset. All the labeled samples in the source domain are used for training. Each setting is repeated for three times and the average results are reported in Table \ref{office31-table}. According to the results, we can see that the performance of the DSAN method is better than that of the S+T method. It is because that the alignment between domains in the DSAN method based on the MMD can improve the classification performance in the target domain. Moreover, the MOML$_\emph{w/o MGDA}$ method performs better than the DSAN method, which means that formulating as a bi-level optimization problem for SSDA can improve the performance. Among all the methods in comparison, the proposed MOML method performs the best, which demonstrates the effectiveness of the proposed MOML framework.

\subsection{Multi-Task Learning}
% Multi-task learning (MTL) \cite{caruana97,zy17b} leverage useful information contained in multiple learning tasks to improve their performance simultaneously.
Multi-task learning (MTL) \cite{caruana97,zy17b} aims to improve the performance of multiple tasks simultaneously by leveraging useful information contained in these tasks. Learning the loss weighting is a challenge in MTL and there are some works \cite{kgc18, sk18, ljd19} to solve this problem.
% However, there are some hyper parameters need to be choosed in  \cite{ljd19}. What is more,
Among those works, the Uncertainty Weighting (UW) method proposed in \cite{kgc18} is only applicable to the square loss as it assumes the Guassian likelihood for the conditional probability that a data point belongs to a class, \citet{sk18} formulate multi-task learning problems from the perspective of multi-objective optimization and implicitly learn the task weights via MGDA, and \citet{ljd19} estimate the task weight of each task as the ratio of the training losses in the last two iterations for the corresponding task.
Different from those works which are all based on single-level optimization problems on the entire training set, we formulate this problem as a multi-objective bi-level optimization problem based on the split of the entire training dataset and solve this problem based on the MOML framework.

\subsubsection{Problem Formulation}
Suppose there are $m$ tasks. The $i$th task has a dataset $\mathcal{D}_i$ for model training. Here each $\mathcal{D}_i$ is partitioned into two subsets: the training dataset $\mathcal{D}^{tr}_i$ and the validation dataset $\mathcal{D}^{val}_i$, where $\mathcal{D}^{tr}_i$ is used to train a multi-task model and $\mathcal{D}^{val}_i$ is to measure the performance of a model. $f(\cdot;\omega)$, the learning function of the multi-task model parameterized by $\omega$, receives data points from the $m$ tasks and outputs predictions. $\alpha_i\in[0,1]$ denotes the loss weight for the $i$th task and $\bm{\alpha}=(\alpha_1,\ldots,\alpha_m)^T$.
The goal is to jointly learn the loss weighting $\bm{\alpha}$ and the model parameter $\omega$. The objective function of the proposed method under the MOML framework is formulated as
\begin{smalleralign}[\footnotesize]
\min_{\bm{\alpha}} &\ \big(\mathcal{L}_{MTL}(\omega^*(\bm{\alpha}),\mathcal{D}^{val}_1),\ldots,\mathcal{L}_{MTL}(\omega^*(\bm{\alpha}),\mathcal{D}^{val}_m)\big)\label{MTL_obj}\\
\mathrm{s.t.} &\ {\omega}^*(\bm{\alpha})=\arg\min_{{\omega}} \sum_{i=1}^m \alpha_i\mathcal{L}_{MTL}(\omega,\mathcal{D}^{tr}_i),\
0\le \alpha_i\le 1\ \forall i,\nonumber
\end{smalleralign}
where $\mathcal{L}_{MTL}(\omega,\mathcal{D})=\frac{1}{|\mathcal{D}|}\sum_{(\mathbf{x},y)\in\mathcal{D}}\ell(f(\mathbf{x;\omega}),y)$ denotes the average loss of $f(\cdot;\omega)$ on a dataset $\mathcal{D}$ with $|\mathcal{D}|$ denoting the size of $\mathcal{D}$ and $\ell(\cdot,\cdot)$ denoting a loss function.
% $\alpha_i\in[0,1]$ denotes the loss weight for the $i$th task,
% Problem \eqref{MTL_obj} is a multi-objective bi-level optimization problem.

\subsubsection{Experiments}
%\subsubsection{Experimental Settings}
Experiments are conducted on the Office-31 and Office-Home \cite{venkateswara2017deep}  datasets. %, which contains 4110 pictures with 31 categories, collected from 3 district domains (tasks): Amazon (\textbf{A}),Webcam (\textbf{W}) and DSLR (\textbf{D}).
%The ResNet-18 \cite{johnson2016perceptual} pretrained on the ImageNet dataset is used as the backbone to extract features. Based on the extracted feature, the multi-task learning model adopts the widely used hard-sharing structure, that is, it has a two-layer fully-connected architecture with the ReLU activation function, where the first layer is shared by all tasks to learn a common representation and the second layer is for task-specific outputs. The model is trained by the Adam optimizer \cite{kingma2014adam} with the learning rate as $0.0001$. We set the training batch size to 115, 20, 32 for each task, respectively and the validation batch size to 32 for all tasks.
Baseline methods in the comparison include Single-Task Learning (STL) \cite{johnson2016perceptual} and Deep Multi-Task Learning (DMTL) with different loss weighting strategies such as Equal Weights (EW), UW \cite{kgc18}, Dynamic Weight Average (DWA) \cite{ljd19} with the temperature parameter $T$ as 2, and MGDA \cite{sk18}. All the MTL models use the same hard-sharing or multi-head network architecture. Experimental settings are put in Appendix \ref{sec:setting_MTL} and experimental results on the Office-31 dataset are put in Appendix \ref{app:office-31}.

\begin{table}[!htbp]
\vskip -0.2in
\caption{Accuracy (\%) on the Office-Home dataset.}
%\vskip 0.1in
\centering
\resizebox{\linewidth}{!}{
\begin{tabular}{|c|c|c|c|c|c|c|c|}
\hline
\multirow{2}{*}{Method} & {Weighting} & \multicolumn{4}{c|}{Task} &\multirow{2}{*}{Avg}\\
\cline{3-6}
~& Strategy &Ar&Cl&Pr&Rw&\\
\hline
 {STL}& -  & 66.48 & 80.64 & \textbf{90.68} & 80.43 & 79.56\\
 \cline{1-7}
 \multirow{4}{*}{DMTL} & EW  & 68.09 & \textbf{80.72} & 89.41 & 80.36 & 79.65\\
 \cline{2-7}
   ~ & UW  & 67.55 & 79.41 & 89.19 & 77.95 & 78.53 \\
 \cline{2-7}
  ~ & DWA  & 65.28 & 79.41 & 89.51 & 79.14 & 78.33 \\
   \cline{2-7}
 ~ & MGDA  & 64.14 & 78.11 & 89.62 & 79.68 & 77.89 \\
  \cline{1-7}
 {MOML}& - & \textbf{69.64} & 80.39 & 90.15 & \textbf{81.08} & \textbf{80.31} \\
\hline
\end{tabular}
}
\label{tbl:mtl-officehome}
\vskip -0.1in
\end{table}

Experimental results on the Office-Home dataset are shown in Table \ref{tbl:mtl-officehome}. According to the results, we can see that MOML outperforms STL and DMTL with different weighting strategies in many cases, which demonstrates the effectiveness of the MOML method. For task \textbf{Pr}, the MOML method is the best among all the MTL methods but all the MTL methods are inferior to STL. This may be because that the learning of the other three tasks (i.e., tasks \textbf{Ar}, \textbf{Cl}, and \textbf{Rw}) hinders the learning of task \textbf{Pr}. However, training the four tasks together can improve the performance of tasks \textbf{Ar} and \textbf{Rw}, which makes MOML outperform STL in tasks \textbf{Ar} and \textbf{Rw}. Moreover, among all the MTL models that learn task weights, MOML is the only one that can outperform the EW strategy in the average sense, which may be beneficial from the bi-level optimization formulation in MOML (i.e., problem \eqref{MTL_obj})  since the validation loss is a more accurate estimation of the generalization loss than the training loss.

\section{Conclusions}

As a generalization of meta learning based on the bi-level formulation, a simple MOML framework based on multi-objective bi-level optimization is proposed in this paper. In the MOML framework, the upper-level subproblem takes multiple objectives of learning problems into consideration. To solve the objective function of the MOML framework, a gradient-based optimization algorithm is proposed and the convergence analysis of this algorithm is studied. Moreover, several use cases of the MOML framework are investigated that demostrates the effectiveness of the MOML framework. In our future work, we will apply the MOML framework to more learning problems.

\bibliography{MOML}
\bibliographystyle{icml2021}

%\appendix
%\section{Do \emph{not} have an appendix here}

%\textbf{\emph{Do not put content after the references.}}
%
%Put anything that you might normally include after the references in a separate supplementary file.

%We recommend that you build supplementary material in a separate document.
%If you must create one PDF and cut it up, please be careful to use a tool that doesn't alter the margins, and that doesn't aggressively rewrite the PDF file.
%pdftk usually works fine.

%\textbf{Please do not use Apple's preview to cut off supplementary material.} In previous years it has altered margins, and created headaches at the camera-ready stage.
%%%%%%%%%%%%%%%%%%%%%%%%%%%%%%%%%%%%%%%%%%%%%%%%%%%%%%%%%%%%%%%%%%%%%%%%%%%%%%%
%%%%%%%%%%%%%%%%%%%%%%%%%%%%%%%%%%%%%%%%%%%%%%%%%%%%%%%%%%%%%%%%%%%%%%%%%%%%%%%
\newpage

\twocolumn[
\icmltitle{Appendix for ``Multi-Objective Meta Learning''}
\vskip 0.3in
]

\begin{appendix}
\section{Additional Definitions} \label{A}

In this section, we give more definitions about convexity of vector-valued function and Kuratowski-Painlev\'{e} set-convergence \cite{lucchetti2006convexity}.
%For $F(x)$ vector valued, it has to be clarified what a minimal solution set is. Whether it is the set of weakly efficient solutions or the set of efficient solutions.

%Let $P = \mathbb{R}^n_{+}$ be a pointed, closed and convex cone. Then this cone $P$ induces a partial order relation $\le_P$ for any two points in $\mathbb{R}^n$, where $x_1 \le_P x_2$ implies that $F_i(x_1)\le F_i(x_2)$. If $x_1 \le_P x_2$ holds, we have $x_2\in x_1 +P$, where $x_1 +P$ denotes a set containing all the $x_2$'s such that $x_2$ satisfies that $x_1 \le_P x_2$. In a MOP, if a point $x^* \in \mathcal{X}$ is called a Pareto optimal solution  or efficient solution, then there is no $x\in \mathcal{X}$ with $F(x)\le_p F(x^*)$ and $F(x)\not = F(x^*)$. Therefore, if $C$ is a subset of $\mathbb{R}^n$, we can give a set notation for Pareto optimal points of the set $C$ by:
%$$\mathrm{Min}\ C := \{y \in C : C\cap (y -  P) = \{y\} \}$$,

%We denote the set of Pareto optimal points of the vector-valued function $F$ by  $\mathrm{Min}F(x)$. Then, the corresponding Pareto solution (efficient solution) set can be defined by:
%\begin{equation*}
%    \mathrm{Eff}(F) := \{x\in \mathcal{X}:F(x)\in \mathop{\mathrm{Min}}_{x\in\mathcal{X}}F(x)\}.
%\end{equation*}
%Then convexity of the vector-valued function in $\mathbb{R}^n$ can also be defined as P-convex from \cite{tammer2003theory}. See Appendix \ref{A}.

\begin{definition}
%$P = \mathbb{R}^m_{+} \subseteq \mathbb{R}^m$ is a pointed, closed and convex cone.
%Then, the interior of the set $P$ is denoted as $\mathrm{int}P$, where
%$$\mathrm{int}P = \{l\in \mathbb{R}^m: l_i>0, \text{ for } i\in \{1,...,m\}\}.$$
%Thus, for $l^1,l^2\in \mathbb{R}^m$, the relation $l^1 \in l^2 + \mathrm{int}P$ implies that %$l^2_i < l^1_i$ for all $i\in \{1,...,m\}$.

%We say $g: \mathbb{R}^n \to \mathbb{R}^m$ is a P-convex function if for every $z_1,z_2\in \mathbb{R}^n$ and for every $\lambda \in [0,1]$
%$$g(\lambda z_1+(1-\lambda)z_2)\in \lambda g(z_1)+(1-\lambda)g(z_2)-P,$$
%and it is strictly P-convex function, if for every $z_1,z_2\in \mathbb{R}^n$, $z_1\not = z_2$ and for every $\lambda \in (0,1)$
%$$g(\lambda z_1+(1-\lambda)z_2)\in \lambda g(z_1)+(1-\lambda)g(z_2)-\mathrm{int}P.$$
$P = \mathbb{R}^m_{+} \subseteq \mathbb{R}^m$ is a pointed, closed and convex cone.  $g: \mathbb{R}^n \to \mathbb{R}^m$ is a P-convex function if for every $z_1,z_2\in \mathbb{R}^n$ and for every $\lambda \in [0,1]$, we have
$$g(\lambda z_1+(1-\lambda)z_2) \le_P \lambda g(z_1)+(1-\lambda)g(z_2).$$
This means the the inequality $g_i(\lambda z_1+(1-\lambda)z_2) \le \lambda g_i(z_1)+(1-\lambda)g_i(z_2)$ holds for all $i\in \{1,...,m\}$, where $g_i(\cdot)$ denotes the $i$th entry in $g(\cdot)$.

$g(z)$ is strictly P-convex function, if for every $z_1,z_2\in \mathbb{R}^n$, $z_1\not = z_2$ and for every $\lambda \in (0,1)$, the inequality
$$g_i(\lambda z_1+(1-\lambda)z_2) < \lambda g_i(z_1)+(1-\lambda)g_i(z_2)$$
holds for all $i\in \{1,...,m\}$.
\end{definition}

\begin{definition} \label{ADD2}
Consider $\{A_n\}$ as a sequence of subsets of an Euclidean space. The set $\mathrm{Li}\  A_n$ is defined as the lower limit of the sequence of sets $\{A_n\}$, that is,
\begin{align*}
    \mathrm{Li}\ A_n:=& \{a\in A: a = \lim_{n\to +\infty}a_n, a_n\in A_n, \\
    & \text{for sufficiently large $n$} \} .
\end{align*}
The set $\mathrm{Ls}\ A_n$ is defined as the upper limit of the sequence of sets $\{A_n\}$, that is,
\begin{align*}
    \mathrm{Ls}\ A_n:= \{&a\in A: a = \lim_{n\to +\infty}a_n, a_n\in A_{n_k}, \\
    & \text{for $n_k$ as a selection of the integers.} \} .
\end{align*}
A sequence $\{A_n\}$ converges in the Kuratowski sense to the set $A$, when
$$\mathrm{Ls}\  A_n \subseteq A \subseteq \mathrm{Li}\  A_n,$$
and we denote this convergence by $A_n\to A$.
\end{definition}

\section{Proofs of Theorems in Section \ref{sec_optimization}}
For the sake of clarity, we firstly introduce some notations from \cite{lucchetti2006convexity}.

The sublevel set of the function $g(z): \mathbb{R}^p \to \mathbb{R}^n$ at height $h \in \mathbb{R}^n$ is defined as
$$g^h := \{z\in \mathbb{R}^p : g(z)\le_P h\}.$$
If $A$ is a closed convex set, then the recession cone of A is defined as
$$0^+(A):=\{d\in\mathbb{R}^p: a+td\in A, \forall a \in A,\forall t \ge 0\}.$$
The recession cone of the sublevel set of the function $g(z)$ is denoted by $H_g$.

We introduce an important concept for vector-valued functions, the weakly minimal point.
\begin{definition}
%A point $l^*$ in a set $C\subseteq \mathbb{R}^m$ is a weakly minimal point if there is no point $l\in C$, such that $l^*\in l+\mathrm{int}P$. Therefore, the set of all minimal points in $C$ w.r.t. the ordering cone $P$ is defined as
%$$\mathrm{WMin} C :=\{l^*\in C: l^*-\mathrm{int}P = \emptyset\}.$$
%Consider a vector-valued function $g(z): \mathbb{R}^n \to  \mathbb{R}^m$ with $z\in\mathcal{Z}$. The set of all weakly minimal points of the vector-valued function $g$ on $\mathcal{Z}$ is denoted by $\mathrm{WMin}\ g(z)$.
Consider a vector-valued function $g(z): \mathbb{R}^n \to  \mathbb{R}^m$ with $z\in\mathcal{Z}$. If a point $z^*$ is a weakly minimal point of $g(z)$, then there is no $z\in \mathcal{Z}$ with $g_i(z)<g_i(z^*)$ for all $i\in\{1,...m\}$. We denote by $\mathrm{WMin}\ g(z)$ the set of weakly minimal points of the vector-valued function $g$ on $\mathcal{Z}$.
Then, the corresponding weakly efficient solution set can be defined as
\begin{equation*}
    \mathrm{WEff}\ (g(z)) := \{z\in \mathcal{Z}:g(z)\in \mathop{\mathrm{WMin}}_{z\in \mathcal{Z}}g(z)\}.
\end{equation*}
\end{definition}

Clearly, for a given function $g(z)$, we have $\mathrm{Min}\ g(z) \subseteq \mathrm{WMin}\ g(z) $. Moreover, if $g$ is strictly P-convex, we have $\mathrm{Min}\ g(z) = \mathrm{WMin}\ g(z) $ and $\mathrm{WEff}\ (g(z))= \mathrm{Eff}\ (g(z))$.

%One would in general only be interested in a weakly efficient point that is also efficient from the application view. However, this concept is very important for theoretical results as optimality conditions.

To prove theorems in Section \ref{sec_optimization}, we prove the following theorems based on the stability analysis of MOPs \cite{lucchetti2004stability}.

\begin{theorem} \label{add1}
$\mathcal{Z}$ is a nonempty closed, convex set in $\mathbb{R}^n$, $g(z): \mathbb{R}^n \to \mathbb{R}^m$ is a vector-valued function with $z\in \mathcal{Z}$. Then if $g_n(z)\to g(z)$ w.r.t. the continuous convergence, we have
$$\mathrm{Ls}\mathrm{WMin}\ g_n(z) \subseteq \mathrm{WMin}\ g(z).$$
\end{theorem}
\begin{proof}
For $l\in \mathrm{LsWMin}\ g_n(z)$, there exists a subsequence $\{l_k\}$ in $\mathrm{WMin}\ g_{n_k}(A)$ such that $l_k\to l$. Here we use $l_{ki}$ to represent the $i$th entry of the vector $l_k$.

We assume that $l\not\in \mathrm{WMin}\ g(z)$. Then there exists $z\in \mathcal{Z}$ such that $g_i(z)< l_i$ for all $i\in\{1,...m\}$. Since $g_n$ continuously converges to $g$, for a sequence $\{z_k\}$ in $A$ satisfying $z_k \to z$, we have $g_{n_k}(z_k)\to g(z)$. Thus, for a sufficiently large $n$, $g_{n_{k},i}(x_k)<l_{ki}$ for all $i\in\{1,...m\}$, where $g_{n_{k},i}(\cdot)$ denotes the $i$th entry in $g_{n_{k}}(\cdot)$. This shows a contradiction with the fact that $l_k \in \mathrm{WMin}\ g_{n_k}(z)$. So $l \in \mathrm{WMin}\ g(z)$ and we reach the conclusion.
\end{proof}
%$l_k \in \mathrm{WMin}\ g_{n_k}(z)$
 %$ g_{n_k}(x_k)<l_k $

\begin{theorem} \label{add2}
$\mathcal{Z}$ is a nonempty closed, convex set in $\mathbb{R}^n$ and $z\in \mathcal{Z}$, $g_n(z)\to g(z)$ w.r.t. the continuous convergence. Then if $g_n(z)$ and $g(z)$ are both P-convex functions and $0^+(A)\cap H_g = \{0\}$, we have
$$\mathrm{Min}\ g(z) \subseteq \mathrm{Li}\mathrm{Min}\ g_n(z).$$
\end{theorem}
\begin{proof}
This results can be directly obtained from Theorems 3.1 and 3.2 of \cite{lucchetti2004stability}.
\end{proof}

\subsection{Proof of Theorem \ref{theorem1}}

\begin{proof}
To show that $F(\omega(\alpha),\alpha)$ is continuous on $\alpha$, we need to prove that for any convergent sequence $\alpha_n\to \bar{\alpha}$, $F(\omega^*(\alpha_n),\alpha_n)$ converges to $F(\omega^*(\bar{\alpha}),\bar{\alpha})$.

%Since $\mathcal{X}$ is compact and $F$ is jointly continuous. There exist minimizers for $F$. For vector-valued function $F$, this means it has Pareto solution set. Let $\{x_n\}$ be a sequence in $\mathcal{X}$ such that $x_n\to \bar{x}$.

Suppose that $\{\alpha_n\}$ is a sequence in $\mathcal{A}$ satisfying $\alpha_n\to \bar{\alpha}$. Since $\mathop{\arg\min}_\omega f(\omega,\alpha)$ is a singleton, we have $\omega^*(\alpha_n) = \mathop{\arg\min}_\omega f(\omega,\alpha_n)$.

Since $\{\omega^*(\alpha)\}$ is bounded for $\alpha\in\mathcal{A}$, there exists a convergent subsequence $\{\omega^*(\alpha_{kn})\}$ such that $\omega^*(\alpha_{kn})\to \bar{\omega}$ for some $\bar{\omega} \in \mathbb{R}^p$.
As $\alpha_{kn}\to \bar{\alpha}$, $\omega^*(\bar{\alpha})$ is the minimizer of the LL objective $f(\omega,\bar{\alpha})$. Therefore, we obtain $\omega^*(\bar{\alpha})= \bar{\omega}$. This means $\{\omega^*(\alpha_{kn})\}$ has only one cluster point $\omega^*(\bar{\alpha})$. Thus, $\omega^*(\alpha_n)$ converges to $\omega^*(\bar{\alpha})$ as $\alpha_n\to \bar{\alpha}$. Because $F$ is jointly continuous, we have $F(\omega^*(\alpha_n),\alpha_n)\to F(\omega^*(\bar{\alpha}),\bar{\alpha})$ as $\alpha_n\to \bar{\alpha}$.
\end{proof}

\subsection{Proof of Theorem \ref{theorem2}}
\begin{proof}
To prove the first claim of Theorem \ref{theorem2}, we firstly show that $\varphi_K(\alpha)$ continuously converges to $\varphi(\alpha)$. Suppose there exists a sequence $\{\alpha_n\}$ in $\mathcal{A}$ satisfying $\alpha_n\to \alpha$. Then for any $\varphi_K(\alpha)$ and sequence $\alpha_n$, we have
\begin{smalleralign}[\footnotesize]
    \Vert \varphi_K(\alpha_n) - \varphi(\alpha)\Vert =& \Vert F(\omega_K(\alpha_n),\alpha_n) -  F(\omega^*(\alpha),\alpha) \Vert \nonumber\\\le& \Vert F(\omega_K(\alpha_n),\alpha_n) -  F(\omega^*(\alpha_n),\alpha_n) \Vert \nonumber\\
    &\ ~ +\Vert F(\omega^*(\alpha_n),\alpha_n) -  F(\omega^*(\alpha),\alpha) \Vert\nonumber
\end{smalleralign}
According to the continuity property in Theorem \ref{theorem1}, we have $F(\omega^*(\alpha_n),\alpha_n) \to  F(\omega^*(\alpha),\alpha)$ as $\alpha_n\to \alpha$. Furthermore, because $F(\cdot,\alpha)$ is uniformly Lipschitz continuous, we have
\begin{smalleralign}[\footnotesize]
    \Vert \varphi_K(\alpha_n) - \varphi(\alpha_n)\Vert =& \Vert F(\omega_K(\alpha_n),\alpha_n) -  F(\omega^*(\alpha_n),\alpha_n) \Vert \nonumber\\
     \le& L \Vert \omega_K(\alpha_n) - \omega^*(\alpha_n) \Vert.\nonumber
\end{smalleralign}
According to Assumption \ref{Ass2}, $\omega_K(\alpha)$ converges to $\omega^*(\alpha)$ uniformly as $K \to +\infty$. Therefore, $\varphi_K(\alpha)$ continuously converges to $\varphi(\alpha)$.

Since $\mathrm{Min}\ \varphi(\alpha) \subseteq \mathrm{WMin}\ \varphi(\alpha)$ and Theorem \ref{add1}, we have the following set relations as
\begin{equation}\label{set1}
    \mathrm{Ls}\mathrm{Min}\ \varphi_K(\alpha) \subseteq \mathrm{Ls}\mathrm{WMin}\ \varphi_K(\alpha)
\subseteq \mathrm{WMin}\ \varphi(\alpha).
\end{equation}
Because $\mathcal{A}$ is a compact convex set in $\mathbb{R}^n$, $0^+(\mathcal{A})=\{0\}$. Then, the condition  $0^+(\mathcal{A})\cap H_{\varphi} = \{0\}$ is naturally satisfied for function $\varphi(\alpha)$. According to Assumption \ref{Ass2}, $\varphi(\alpha)$ and $\varphi_K(\alpha)$ are both P-convex functions. Then we obtain the lower part of the set convergence from Theorem \ref{add2} as
\begin{equation}\label{set2}
\mathrm{Min}\ \varphi(\alpha) \subseteq \mathrm{Li}\mathrm{Min}\ \varphi_K(\alpha)
\subseteq \mathrm{Li}\mathrm{WMin}\ \varphi_K(\alpha).\end{equation}
Because $\varphi(\alpha)$ is strictly P-convex, we have $\mathrm{WMin}\ \varphi = \mathrm{Min}\ \varphi$ and then we get $\mathrm{Min}\ \varphi_K(\alpha) \to \mathrm{Min}\ \varphi(\alpha)$ according to Definition \ref{ADD2}.

For the second claim, let $\alpha_n \in \mathrm{Eff}\ \varphi_K(\alpha)$ and $\alpha_n \to \bar{\alpha}$. Since $\mathrm{Min}\ \varphi_K(\alpha) \to \mathrm{Min}\ \varphi(\alpha)$, we get $\varphi_K(\alpha_n) \to \varphi(\bar{\alpha})$ and $\bar{\alpha} \in \mathrm{Min}\ \varphi(\alpha)$, which implies $\mathrm{LsEff}\ \varphi_K(\alpha) \subseteq \mathrm{Eff}\  \varphi(\alpha)$.

For the lower limit, by defining $\bar{\alpha} \in \mathrm{Eff}\ \varphi(\alpha)$, the corresponding minimal point satisfies $\bar{l}=\varphi(\bar{\alpha})\in \mathrm{Min}\ \varphi(\alpha)$. Based on the first claim of this theorem, there exists a sequence $\{l_K\}$ in $\mathrm{Min}\ \varphi_K(\alpha)$ such that $l_K \to \bar{l}$. Then we can take a bounded sequence $\{\alpha_K\}$, where $\alpha_K = \varphi_K^{-1}(l_K)$ and the subsequence of $\{\alpha_K\}$ has a cluster point. Because $\varphi(\alpha)$ is strictly P-convex, this cluster point is $\bar{\alpha}$. Then, we have $\alpha_K \to \bar{\alpha}$, which implies $\mathrm{Eff}\ \varphi(\alpha) \subseteq \mathrm{LiEff}\  \varphi_K(\alpha)$. Combined with the upper limit convergence, we can get $\mathrm{Eff}\ \varphi_K(\alpha) \to \mathrm{Eff}\ \varphi(\alpha)$.
\end{proof}

In fact, if we consider the weakly minimal points under Assumptions \ref{Ass1} and \ref{Ass2}, we can still obtain similar convergence results to those in Theorem \ref{theorem2}, i.e.,
\begin{smalleralign}[\footnotesize]
\mathrm{WMin}\ \varphi_K(\alpha) \to \mathrm{WMin}\ \varphi(\alpha),\ \nonumber
\mathrm{WEff}\ \varphi_K(\alpha) \to \mathrm{WEff}\ \varphi(\alpha).   \nonumber
\end{smalleralign}
Since $\varphi(\alpha)$ is strictly P-convex, the first claim can be directly obtained from the set relations in Eqs. \eqref{set1} and \eqref{set2}. Then, the proof of the convergence of the weakly efficient solution follows that of Theorem \ref{theorem2}.

\section{Experimental Settings for Use Cases of MOML}

\subsection{Few-Shot Learning}
\label{sec:setting_FSL}

Experiments are conducted on two FSL benchmark datasets, CUB-200-2011 (referred to as CUB) \cite{wah2011caltech} and \textit{mini}-ImageNet \cite{vinyals2016matching}. The CUB dataset contains 200 classes and 11,788 images in total. Following \cite{hilliard2018few}, we randomly split this dataset into a base dataset containing 100 classes, a validation dataset containing another 50 classes, and a novel dataset containing the rest 50 classes. The \textit{mini}-ImageNet dataset contains 100 classes with 600 images per class, sampling from the ImageNet dataset \cite{deng2009imagenet}. By following \cite{ravi2016optimization}, this dataset is partitioned into 64, 16, and 20 classes for the base, validation, and novel datasets, respectively.

For both MAML and MOML methods, each task is a 5-way $k$-shot classification problem, where $k=1$ or $5$. The input images are resized to $84\times 84$ for both two datasets and applied data augmentation including random crop, random horizontal flip, and color jitter. A four-layer convolutional neural network (Conv-4) is used as the backbone, which consists of four blocks each of which consists of a convolution layer with 64 kernels of size $3 \times 3$, stride 1, and zero padding, a batch normalization layer, a ReLU activation function, and a max-pooling layer with the pooling size $2\times 2$. After the backbone, a linear layer with 5 neurons is used as a classifier to output the prediction for the input image. The Adam optimizer \cite{kingma2014adam} with the learning rate $0.001$ is used.

In the meta training, we randomly sample $k$ and 16 instances per class as the support set and the query set, respectively, in each episode. %and the rest in each class are used as the support set.
The adversarial attack on the query set is performed by the PGD attack with a perturbation size $\epsilon =2/255$ and it takes $7$ iterative steps with the step size of $2.5\epsilon$. In the meta testing, we generate 600 5-way $k$-shot tasks from $\mathcal{D}_{novel}$, where each task has $k$ samples for the adaptation and 16 samples for testing. The final results is the average on all the 600 testing tasks. We compare with MAML since problem (\ref{eq:fsl}) can reduce to MAML when there is only the first objective in its UL subproblem.

\subsection{NAS}
\label{sec:setting_NAS}

The search space and training procedure of MOML adopt the same settings as DARTS \cite{lsy19}. Specifically, in both normal and reduction cells, the set of operations $\mathcal{O}$ contains eight operations, including $3\times3$ separable convolutions, $5\times5$ separable convolutions, $3\times3$ dilated separable convolutions, $5\times5$ dilated separable convolutions, $3\times3$ max pooling, $3\times3$ average pooling, identity, and zero. Half of the training set is used for training a model, and the other half is for the validation. A small network of 8 cells is trained with the batch size as 64 and 16 initial channels for 50 epochs. The Adam optimizer \cite{kingma2014adam} with the learning rate $3\times 10^{-4}$, the momentum $\beta=(0.5, 0.999)$, and the weight decay $1\times 10^{-3}$ is used to update $\bm{\alpha}$ in the UL subproblem. The SGD optimizer with the decayed learning rate down from $0.025$ to $0$ by a cosine schedule, the momentum $0.9$, and the weight decay $3\times 10^{-4}$ is used to update $\omega$ in the LL subproblem.

In the evaluation stage, a neural network of 20 searched cells is trained on the full training set for 600 epochs with the batch size as 96, the initial number of channels as 36, the length of a cutout as 16, the dropout probability as 0.2, and auxiliary towers of weight as 0.4. The full testing set is used for testing.  Adversarial examples are generated using the PGD attack with the perturbation size $\epsilon = 1/255$ and the PGD attack takes 10 iterative steps with the step size of $2.5\epsilon$ as suggested in \cite{pgd17}.

\subsection{Semi-Supervised Domain Adaptation}
\label{sec:setting_SSDA}

We use the ResNet-50  model \cite{he2016deep} pretrained on the ImageNet dataset as the backbone network followed by a Fully-Connected (FC) layer. %with the tanh activation function.
The same network structure is used for all baseline methods. For all experiments, the SGD optimizer with the learning rate $0.001$, the momentum $0.9$ and the weight decay $5\times 10^{-4}$ is used for optimization. The batch size is set to 96, including 32 images in the source, labeled target, and unlabeled target domains, respectively. %To suppress noisy activation at the early stages of training, instead of fixing the weight of the alignment loss, we gradually change it from 0 to 1 by a progressive schedule.
%The learning rate of the classification network is set to 0.01. % which is 10 times that of the backbone network.

\subsection{Multi-Task Learning}
\label{sec:setting_MTL}

% The ResNet-18 \cite{johnson2016perceptual} pretrained on the ImageNet dataset is used as the backbone to extract features. Based on the extracted feature, the multi-task learning model adopts the widely used hard-sharing or equivalently multi-head structure, that is, it has a two-layer fully-connected architecture with the ReLU activation function, where the first layer is shared by all tasks to learn a common representation and the second layer is for task-specific outputs. The model is trained by the Adam optimizer \cite{kingma2014adam} with the learning rate as $0.0001$. Both the Office-Home and Office-31 datasets are split into three parts, including 60\% for training, 20\% for validation, and the remaining 20\% for testing. For the Office-31 dataset, since different tasks have different numbers of data points, to make the entire training data of each task be visited in the same number of iteration, the training batch size is set to 115, 20, 32 for each task, respectively, and the validation batch size is set to 32 for all tasks. Similarly, for the Office-Home dataset, we set the training batch size to 35, 64, 65, 64 for each task, respectively, and the validation batch size is set to 32 for all tasks.

The ResNet-50 pretrained on the ImageNet dataset is used as the backbone to extract features. Based on the extracted features, the multi-task learning model adopts the widely used hard-sharing or equivalently multi-head structure, that is, it has a two-layer fully-connected architecture with the ReLU activation function, where the first layer is shared by all tasks to learn a common representation and the second layer is for task-specific outputs. The model is trained by the Adam optimizer \cite{kingma2014adam} with the learning rate as $0.0001$. Both the Office-Home and Office-31 datasets are split into three parts, including 60\% for training, 20\% for validation, and the remaining 20\% for testing.
For the Office-31 dataset, we set the training batch size and the validation batch size to 32 for all tasks. For the Office-Home dataset, we set the training batch size to 16 and the validation batch size to 32 for all tasks.

\section{Ablation Study on NAS} \label{app:nas}

Here we compare MOML with a variant of the MOML method by replacing the MOP in the UL problem with a linearly combined single-objective problem with equal weights $1$, which is denoted by MOML$_\emph{w/o MGDA}$. For MOML$_\emph{w/o MGDA}$, we adopt the same experimental settings as the MOML method. The comparison results are shown in Table \ref{tbl:nas_ablation}. When $L$ equals 1 or 2, MOML$_\emph{w/o MGDA}$ searches smaller architectures than MOML. To make the network size searched by MOML$_\emph{w/o MGDA}$ comparable with that of MOML, we use different numbers of initial channels in MOML$_\emph{w/o MGDA}$ during the evaluation process. Compared with MOML, the MOML$_\emph{w/o MGDA}$ method with a comparable size has lower clean accuracy and robustness in most cases, which demonstrates the effectiveness of the MGDA used.

\begin{table}[!htbp]
\caption{Comparison between MOML and MOML$_\emph{w/o MGDA}$ on the CIFAR-10 dataset. $\uparrow$ indicates that a larger value is better, while $\downarrow$ implies that a lower value is better. ``$\{\text{MOML$_\emph{w/o MGDA}$-V\#size-C\#channels}\}$" means that the architecture searched by MOML$_\emph{w/o MGDA}$ with $L$ as ``size'' is evaluated by the initial number of channels as ``channels''. ``$\{\text{MOML-V}\#\text{size}\}$" denotes the architecture searched by MOML with $L$ as ``size''.}
\resizebox{\linewidth}{!}{
\begin{tabular}{lccccc}
\toprule
\multirow{2}*{\textbf{Architecture}} & \textbf{Params} & \textbf{Clean Acc.}  & \textbf{PGD Acc.}\\
~ & \textbf{(MB)}~$\downarrow$ & \textbf{(\%)}~$\uparrow$ & \textbf{(\%)}~$\uparrow$ \\
\midrule
MOML$_\emph{w/o MGDA}$-V1-C38    &1.750  &96.36  &40.20 \\
MOML$_\emph{w/o MGDA}$-V2-C42    &2.402  &97.03  &31.44 \\
MOML$_\emph{w/o MGDA}$-V3-C36    &3.018  &97.18  &35.36 \\
\midrule
MOML-V1    &1.754  &96.48  &42.66 \\
MOML-V2    &2.367  &97.18  &36.15 \\
MOML-V3    &3.018  &97.25  &35.22 \\
\bottomrule
\end{tabular}
}
\label{tbl:nas_ablation}
\end{table}

\section{Experimental Results on the Office-31 Dataset for Multi-Task Learning} \label{app:office-31}

Experimental results on the Office-31 dataset are shown in Table \ref{tbl:mtl-office31}. According to the results, we can see that MOML outperforms STL and DMTL with different weighting strategies in most cases, which demonstrates the effectiveness of the MOML method. For task \textbf{A}, the MOML method is the best among all the MTL methods but all the MTL methods are inferior to STL. This may be because that the learning of the other two tasks (i.e., tasks \textbf{W} and \textbf{D}) hinders the learning of task \textbf{A}. However, training the three tasks together can improve the performance of task \textbf{D}, which makes all MTL methods outperform STL in task \textbf{D}. Moreover, among all the MTL models, MOML is the only one which performs not worse than DMTL with the EW strategy in each task, which may be beneficial from the bi-level optimization formulation in MOML which uses the validation loss to estimate the generalization loss more accurately.

\begin{table}[!htbp]
\caption{Accuracy (\%) on the Office-31 dataset for multi-task learning.}
%\vskip 0.1in
\centering
\resizebox{\linewidth}{!}{
\begin{tabular}{|c|c|c|c|c|c|c|}
\hline
\multirow{2}{*}{Method} & {Weighting} & \multicolumn{3}{c|}{Task} &\multirow{2}{*}{Avg}\\
\cline{3-5}
~& Strategy &A&D&W&\\
\hline
 {STL}& -  & \textbf{89.06} & 96.72 & 98.89 & 94.89 \\
 \cline{1-6}
 \multirow{4}{*}{DMTL} & EW  & 87.35 & \textbf{99.18} & 98.89 & 95.14 \\
 \cline{2-6}
   ~ & UW  & 86.50 & 97.54 & {97.78} & 93.94 \\
 \cline{2-6}
  ~ & DWA  & 86.67 & \textbf{99.18} & 97.22 & 94.36 \\
   \cline{2-6}
 ~ & MGDA  & {81.88} & 97.54 & 98.89 & 92.77 \\
  \cline{1-6}
 {MOML}& - & 88.03 & \textbf{99.18} & \textbf{99.44} & \textbf{95.55} \\
\hline
\end{tabular}
}
\label{tbl:mtl-office31}
\end{table}

\end{appendix}
\end{document}